\documentclass{article}
\usepackage[utf8]{inputenc}
\usepackage{amsthm}
\usepackage{amsmath}
\usepackage{amssymb}
\usepackage{bbm}
\usepackage{url}
\usepackage{natbib}
\usepackage{graphicx}
\usepackage{booktabs}
\usepackage{xcolor}
\usepackage{bbm}
\usepackage{subcaption}
\usepackage{todonotes}
\usepackage[export]{adjustbox}
\usepackage[margin=1in]{geometry}
\usepackage{float}
\usepackage{nopageno}
\usepackage{pdfpages}

\theoremstyle{plain}
\newtheorem{thm}{\protect\theoremname}
\theoremstyle{plain}
\newtheorem{lem}[thm]{\protect\lemmaname}
\theoremstyle{plain}

\makeatother

\providecommand{\corollaryname}{Corollary}
\providecommand{\lemmaname}{Lemma}
\providecommand{\theoremname}{Theorem}

\newcommand{\R}{\mathbb{R}}
\newcommand{\Lbf}{L_{\mathrm{BF}}}
\newcommand{\Ldyntau}{L_{\mathrm{Dyn},\boldsymbol{\tau}}}
\newcommand{\Lq}{L_{Q}}
\newcommand{\Lqwalk}{L_{Q_{\boldsymbol{\tau}, \boldsymbol{\beta}, \nu}^{(W)}}}
\newcommand{\Llocked}{L_{\mathrm{locked}}}
\newcommand{\LMarBLR}{L_{\mathrm{MarBLR}}}
\newcommand{\Lblr}{L_{\mathrm{BLR}}}

\newcommand{\taulock}{\boldsymbol{\tau}_{\mathrm{locked}}}
\newcommand{\thetataulock}{\tilde{\theta}_{\tau_{\mathrm{locked}}}}
\newcommand{\Qtaumusigma}{Q_{\boldsymbol{\tau},\boldsymbol{\mu},\Sigma}}
\newcommand{\Qtaulocked}{Q_{\boldsymbol{\tau}_{\mathrm{locked}},\tilde{\theta}_{\tau_{\mathrm{locked}}},\epsilon^2\Sigma_{\mathrm{init}}}}
\newcommand{\Qwalk}{Q_{\boldsymbol{\tau}, \boldsymbol{\beta}, \nu}^{(W)}}
\newcommand{\Qsub}{Q^{\mathrm{sub}}}
\newcommand{\psub}{p_0^{\mathrm{sub}}}

\newcommand{\thetainit}{\theta_{\mathrm{init}}}
\newcommand{\Sigmainit}{\Sigma_{\mathrm{init}}}

\DeclareMathOperator{\KL}{KL}
\DeclareMathOperator{\E}{E}

\DeclareMathOperator{\Bernoulli}{Bernoulli}
\DeclareMathOperator{\Tr}{Tr}

\usepackage{wasysym}

\title{Bayesian logistic regression for online recalibration and revision of risk prediction models with guarantees}
\author{Jean, Romain, Alexej, Berkman}

\begin{document}



# Bayesian logistic regression for online recalibration and revision of risk prediction models with performance guarantees

Authors:
Jean Feng
*Department of Epidemiology and Biostatistics*
*University of California, San Francisco*
*San Francisco, California, USA*

Alexej Gossmann
*CDRH-Center for Devices and Radiological Health*
*Food and Drug Administration*
*Silver Spring, MD, USA*

Berkman Sahiner
*CDRH-Center for Devices and Radiological Health*
*Food and Drug Administration*
*Silver Spring, MD, USA*

Romain Pirracchio
*Department of Anesthesia and Perioperative Care*
*University of California, San Francisco*
*San Francisco, California, USA*

Corresponding Author:
Jean Feng
*550 16th Street*
*San Francisco CA 94158*
*jean.feng@ucsf.edu*

## ABSTRACT

Objective: After deploying a clinical prediction model, subsequently collected data can be used to fine-tune its predictions and adapt to temporal shifts. Because model updating carries risks of over-updating/fitting, we study online methods with performance guarantees.

Materials and Methods:
We introduce two procedures for continual recalibration or revision of an underlying prediction model: Bayesian logistic regression (BLR) and a Markov variant that explicitly models distribution shifts (MarBLR). We perform empirical evaluation via simulations and a real-world study predicting COPD risk. We derive "Type I and II" regret bounds, which guarantee the procedures are non-inferior to a static model and competitive with an oracle logistic reviser in terms of the average loss.

Results: Both procedures consistently outperformed the static model and other online logistic revision methods. In simulations, the average estimated calibration index (aECI) of the original model was 0.828 (95%CI 0.818-0.938). Online recalibration using BLR and MarBLR improved the aECI towards the ideal value of zero, attaining 0.265 (95%CI 0.230-0.300) and 0.241 (95%CI 0.216-0.266), respectively. When performing more extensive logistic model revisions, BLR and MarBLR increased the average AUC (aAUC) from 0.767 (95%CI 0.765-0.769) to 0.800 (95%CI 0.798-0.802) and 0.799 (95%CI 0.797-0.801), respectively, in stationary settings and protected against substantial model decay. In the COPD study, BLR and MarBLR dynamically combined the original model with a continually-refitted gradient boosted tree to achieve aAUCs of 0.924 (95%CI 0.913-0.935) and 0.925 (95%CI 0.914-0.935), compared to the static model's aAUC of 0.904 (95%CI 0.892-0.916).

Discussion: Despite its simplicity, BLR is highly competitive with MarBLR. MarBLR outperforms BLR when its prior better reflects the data.

Conclusions: BLR and MarBLR can improve the transportability of clinical prediction models and maintain their performance over time.

## BACKGROUND

A growing number of prediction models have been validated and approved as clinical decision support systems and medical diagnostic devices [1]. Models that have been successfully deployed need to be regularly monitored and updated over time, because locked algorithms are known to decay in performance due to changes in clinical practice patterns, patient case mix, measurement procedures, and more [2–4].

With the expansion of electronic health record systems, we have a unique opportunity to embed models that continuously learn and evolve by analyzing streaming medical data, which are often referred to as online learning or continual learning systems [5–8]. Online learning systems not only have the potential to protect against the consequences of distributional shifts over time, but they may also improve prediction performance, e.g. by increasing precision of their estimates or personalizing predictions to local medical practices [9,10]. Nevertheless, there are major technical challenges in developing reliable online learning systems [8], so establishing their safety and effectiveness is of utmost importance [11].

In this study, we focus our attention on online model revision for risk prediction models, in which data are revealed in a sequential manner and the goal of the online model is to dynamically predict the probability of having or developing a disease (or outcome) given forecasted scores from an underlying model. We build on the common practice of using logistic regression to recalibrate and/or revise forecasted scores from an underlying prediction model on an initial dataset [2] and extend it to the online setting with streaming labelled data (Figure 1). In the simplest case, the revisions only depend on the forecasted score, which is also known as online model recalibration. In more complex situations, the updated prediction can depend on both the forecasted score and other patient variables. We are particularly interested in procedures that can safely update ``black-box'' models such as gradient boosted trees and neural networks, which have achieved unprecedented success by

capturing nonlinearities and interactions in the data. Simply refitting black-box models on accumulating data may carry risks because the refitted version is not guaranteed to outperform the original model [12]. However, we *can* analyze the theoretical properties of logistic model revision, even for an underlying black-box prediction model.

Methods for continually updating clinical prediction models have traditionally relied on dynamic Bayesian models [13,14] or online hypothesis testing [15]. However, these methods do not provide theoretical guarantees under model misspecification or distributional shifts. More recently, game-theoretic online learning methods have been applied to the problem of online model recalibration and provide performance guarantees that bound its cumulative loss relative to some oracle procedure (also known as "regret"). Kuleshov et al. [16] use a non-parametric binning technique, but this method converges slowly in practice and cannot be used to revise a model with respect to patient variables. Davis et al. [17] apply Adam to the related problem of estimating dynamic calibration curves [18]; however, recent theoretical results show that the optimal regret bound for online logistic regression is, in fact, achieved by Bayesian model updating [19,20].

## OBJECTIVE

In this work, we develop online revision methods for "black-box" models that may be locked or evolving over time and provide theoretical guarantees without making any assumptions about the data distribution or the quality of the underlying model. We investigate online model revision using Bayesian logistic regression (BLR) and BLR with a Markov prior that explicitly models distribution shifts (MarBLR). To quantify the safety and effectiveness of the proposed online model revisers, we introduce the notions of Type I and II regret. We derive regret bounds for BLR and MarBLR, which provides a recipe for selecting a Bayesian prior that satisfies desired performance guarantees. In simulation studies, we evaluate BLR and MarBLR for online model recalibration and revision and as wrappers for black-box refitting procedures. We then apply the online updating procedures to Chronic Obstructive Pulmonary Disease (COPD) risk prediction in a retrospective dataset from 2012 to 2020. Code is available at http://github.com/jjfeng/bayesian_model_revision.

## MATERIALS AND METHODS

Because the nature of future distribution shifts is typically unknown, we study the safety and effectiveness of an online model reviser in the presence of *arbitrary* distribution shifts, following the game-theoretic online learning literature. This general framework allows us to study, for example, distribution shifts induced by the deployment of the ML model itself, which has been raised as a concern in a number of recent works [21,22].

**A Framework for Evaluating Online Model Revision Algorithms**

Denote patient variables with $x \in \mathcal{X}$ and binary outcomes with $y$. Suppose the streaming data is received at discrete times $t = 1, \ldots, T$. At time $t$, we observe a new observation $(x_t, y_t)$. Let $\widehat{f}_t: \mathcal{X} \mapsto \mathbb{R}$ denote the underlying prediction model. Let $\widehat{A}_t: \mathbb{R} \times \mathcal{X} \mapsto [0,1]$ be the model revision that outputs a probability. The initial clinical prediction model is defined by the

composition $\widehat{A_1} \circ \widehat{f_1}$. If $\widehat{f_1}$ is well-calibrated, we may simply define $\widehat{A_1}$ to be the identity function; otherwise, $\widehat{A_1}$ should be estimated on an initial recalibration dataset [2]. Let $\tau = (\tau_1, \tau_2, \ldots \tau_s)$ be any sequence of $s$ times in which the model revision is updated. In certain cases, one may observe a batch of observations at each time point instead. We discuss how the theoretical framework and results need to be adjusted to handle batched data in the Appendix.

The online learning procedure can be described as follows. For time steps $t = 1, 2, I, T$:
1. Patient $x_t$ is revealed. The online reviser deploys modification $\widehat{A_t}$ for model $\widehat{f_t}$ and releases a prediction for the patient.
2. We observe the respective outcome $y_t$.
3. The evolving model selects $\widehat{f_{t+1}}$. The online reviser selects $\widehat{A_{t+1}}$. The next observation $(x_{t+1}, y_{t+1})$ is acquired, but not revealed to the model yet.

In this setup, we do not make any assumptions about how the data is generated or the reliability of the underlying model.

For the theoretical analyses, we quantify the performance of the online model reviser by its average over the entire time period, i.e. $-\frac{1}{T}\sum_{t=1}^{T} \log p\left(y_t, \widehat{A_t}(\widehat{f_t}(x_t), x_t)\right)$ where $-\log p(y, \hat{p})$ denotes the negative log likelihood for the outcome $y$ and predicted probability $\hat{p}$. This is a direct extension of offline logistic regression—which is usually fit using maximum likelihood estimation—to the online setting.

For an online model reviser to be safe, it should be, at the very least, non-inferior to locking the original model. Drawing analogy to the hypothesis testing literature, locking the model can be viewed as the "null" hypothesis and using the online model reviser as the "alternative." Type I error is then the incorrect rejection of the null. Combining this with the notion of regret from the online learning literature, we define Type I regret as the average increase in the loss when using the online reviser instead of the original model, i.e.
$(-\frac{1}{T}\sum_{t=1}^{T}[\log p\left(y_t, \widehat{A_t}(\widehat{f_t}(x_t), x_t)\right) - \log p\left(y_t, \widehat{A_1}(\widehat{f_1}(x_t), x_t)\right)])_+.$

Type I regret spans the non-negative real values, where a smaller value is better. A safe online model reviser should control its value below some pre-specified non-inferiority margin $\gamma > 0$. Over-zealous model updating (also known as over-updating [2]) tends to inflate Type I regret. Nevertheless, it is not enough to solely control Type I regret because locking the original model perfectly controls Type I regret without offering any protection against distribution shifts.

In addition, we quantify the effectiveness of an online model reviser by comparing its performance to that achieved by the best sequence of model revisions in retrospect. More specifically, if one had access to observations for times $t = 1, \ldots, T$, there is some oracle sequence of model revisions $\{A^*_{\tau,t}: t = 1, \ldots, T\}$, restricted to update times $\tau$, that minimizes the average loss. We refer to an oracle as static when $\tau$ is the empty set (i.e. the model revision sequence is the optimal constant sequence) and dynamic otherwise. We define Type II $\tau$-regret as the average performance difference between the online reviser and this dynamic oracle, i.e.
$(-\frac{1}{T}\sum_{t=1}^{T}[\log p\left(y_t, \widehat{A_t}(\widehat{f_t}(x_t), x_t)\right) - \log p\left(y_t, A^*_{\tau,t}(\widehat{f_t}(x_t), x_t)\right)])_+.$

Type II regret is large when we fail to update the model fast enough. It is especially large when we make a Type II error---when we fail to reject the "null" hypothesis---and don't update the

model at all. Similar to Type I regret, Type II regret spans the non-negative real values where small values are better.

Our aim is to design online model revisers that minimize Type II regret while controlling Type I regret, *regardless of how the data distribution and underlying prediction model change over time*. There is a trade-off between Type I and II regret, since increasing the frequency and magnitude of the model revision updates typically increase Type I regret but decrease Type II regret.

**Bayesian Logistic Revision and a Markov Variant**

BLR and MarBLR perform inference for logistic model revisers of the form

$$\frac{1}{1 + exp\left(-\theta_{t,0} - \theta_{t,1}^\top z(\widehat{f}_t(x), x)\right)}$$

where $z$ is some basis expansion of the score from the underlying model and patient variables. Let $\theta_t$ denote the logistic revision parameters at time $t$. After receiving a new labeled observation, we update the posterior for the model revision parameters according to Bayes' theorem. We predict the probability that $Y = 1$ for a patient variables $x$ using the posterior mean of $\Pr(Y = 1|X = x)$.

BLR defines a Gaussian prior $N(\theta_{init}, \Sigma_{init})$ over the model revision parameters and assumes the parameters are fixed over the entire time period, i.e. the true $\theta_t$ do not have time dependence. Because BLR is an over-simplification of the data, we also consider a generalization of BLR called MarBLR that allows for updates to the revision parameters over time (Figure 2). In particular, MarBLR supposes the revision parameters need updating at each time point with a prior probability of α and evolves according to the Gaussian random walk $\theta_t = \theta_{t-1} + V_t W_t$, where $V_t$ are independent $N(0, \Sigma_t)$ random vectors and $W_t$ is a binary random variable with success probability α. Thus, α is the prior probability for how frequently the model revision needs to be updated and $\Sigma_t$ is our prior regarding the magnitude of these updates. In the theoretical analyses, we analyze MarBLR with $\Sigma_t = \delta^2 \Sigma_{init}$. In practice, we choose $\Sigma_t = \delta^2 \widehat{\Sigma_{t-1}}$ where $\widehat{\Sigma_{t-1}}$ is the posterior covariance matrix of the $\theta_{t-1}$ at time $- - 1$. MarBLR reduces to BLR when α $= 0$ or $\delta^2 = 0$.

We derive Type I and II regret bounds for BLR and MarBLR by extending [19], which only compared BLR to a static oracle. We generalize the results to handle dynamic oracles and the more general MarBLR procedure. The Type I regret bounds hold as long as BLR and MarBLR are able to revert to the original model. Thus, one should always choose the basis expansion $z$ to include scores from both the original and evolving models $f_1$ and $f_t$, respectively.

Bayesian inference for MarBLR requires marginalizing over $2^T$ possible update times. Enumerating all possible shift times is computationally intractable and because the posterior does not have a closed form, we use instead a Laplace approximation of the logistic posterior [23] and perform Kalman filtering with collapsing [13,24] (see Section A in the Appendix).

**Simulation Studies**

We assess the performance of BLR and MarBLR in three simulation studies with increasingly complex model revisions. The theoretical guarantees for BLR and MarBLR are designed to hold in all three settings.

**Data Generating Procedure.** We generate 10-dimensional patient variables $X$ using a multivariate normal distribution and binary outcomes $Y$ using a logistic model. We introduce distribution shifts by perturbing coefficients of this model. The underlying prediction model is a gradient-boosted tree (GBT). To ensure that the initial prediction model is well-calibrated, we fit $\widehat{A_1}$ on 100 observations held out from the original training data. Thereafter, we observe a single subject at each time point and run the procedure for $T = 500$ time steps. For each simulated condition, we perform 50 replicates to estimate standard errors. (Additional simulation details are in the Appendix.)

*Scenario 1* is on online recalibration of a locked underlying model in the presence of temporal shifts as well as recalibration across patient subgroups to see if these methods can improve algorithmic fairness. Our motivation is based on recent works that highlight disparities in the performance of ML models when patient populations are heterogeneous and unbalanced [25]. We define patient subgroups A and B with 20% and 80% prevalence, respectively. We simulate an initial distribution shift in each subgroup and introduce one and two subsequent shifts in subgroups A and B, respectively. We fit a univariate logistic recalibration that ignores subgroup status as well as a subgroup-aware recalibration using interaction terms between the forecasted score and subgroup status.

*Scenario 2* is on learning logistic model revisions where the inputs are the forecasted score from a locked underlying model and ten patient variables, i.e. $\frac{1}{1+exp(\theta_0 + \theta_1 \widehat{f_1}(x) + \theta_2^\top x)}$ where $\theta_0$ is the intercept, $\theta_1$ is the coefficient for the locked model, and $\theta_2$ is a vector of coefficients associated with the patient variables. We consider three types of data distributions: independent and identically distributed (IID) data after an initial distribution shift (Initial Shift), cyclical variation (Cyclical), and repeated dataset shifts that lead to gradual performance decay of the original model (Decay). The initial GBT is trained on 300 observations.

*Scenario 3* performs online ensembling of the original model and a black-box refitting procedure $\widehat{f}_t$. The logistic revision is of the form

$$\frac{1}{1 + exp\left(\theta_0 + \theta_1 \widehat{f_1}(x) + \theta_2 \widehat{f_t}(x)\right)}$$

where $\theta_1$ and $\theta_2$ are the coefficients for the locked and evolving models, respectively. We use the Initial Shift and Decay data settings from Scenario 2 and the same initial model. We simulate a reliable black-box refitting procedure by refitting on all available data (All-Refit). To test if the online model reviser can protect against black-box refitting procedures with unreliable performance, we simulate an evolving model that refits on the most recent 75 observations up to time 100 and then suddenly refits using only the most recent 30 observations thereafter (Subset-Refit). While this refitting procedure is unlikely to be used in practice, it lets us simulate how the online model reviser might respond to sudden deterioration in the evolving model.

**Performance Metrics.** For baseline comparison, we lock the original model, perform online logistic revision using Adam, and cumulative logistic regression (CumulativeLR) as suggested in [19]. Briefly, Adam performs a gradient-based update to the current logistic revision parameters and CumulativeLR refits logistic revision parameters by minimizing with respect to all prior observations. BLR is similar to CumulativeLR in nature but appropriately integrates over uncertainty using a Bayesian framework, and MarBLR additionally models distribution shifts. We evaluate the methods in terms of the expected negative log likelihood (NLL), estimated calibration index (ECI) [26], and area under the receiver-operating characteristic curve (AUC), when appropriate. Note that a model is better if it has smaller NLL and ECI and bigger AUC, where the minimum ECI is zero and the maximum AUC is 1. We use aNLL, aECI, and aAUC to denote their average value over the time period.

**Hyperparameter Selection.** We select the Gaussian prior for BLR and MarBLR such that its mean is the estimated logistic revision parameters on the initial recalibration dataset and its covariance is a scaled version of the standard error matrix to achieve the desired Type I regret control. In the third simulation, we must construct an initial recalibration dataset such that the fixed and evolving models are not exactly the same. We do this by representing the evolving model with a model that was trained on 90% of the original training data. The Gaussian prior at $t = 1$ is then centered at the estimated revision parameters with the constraint that the coefficient for the evolving model is zero and the prior covariance is a scaled version of the Hessian matrix.

The priors for BLR and MarBLR are selected such that their Type I regret is no more than 5% of the initial loss of the original locked model in the first and third scenarios. Because the second scenario considers higher dimensional model revisions and Type I regret scales with the dimensionality of the problem (see Theoretical Results in the following section), we use a looser bound of 10% in this setting. Because the regret bounds for Adam and CumulativeLR are too wide to be meaningful, these methods are run without any Type I regret control.

**COPD Dataset**

We analyzed 108002 in-patient admissions to UCSF Health from 2012-06-15 to 2020-12-01, of which 2756 admissions resulted in a primary or secondary diagnosis of COPD (based on ICD-9 and ICD-10 codes). We ordered observations by their admission dates. There are a total of 36 predictors available, including age, history of smoking, history of COPD, active outpatient medications prior to emergency department (ED) presentation, and medications administered in the ED prior to the point of admission. The initial model is fit using a gradient boosted tree on the first 2500 observations, of which 92 are positive cases. For initial model recalibration, we use the 1000 immediately following observations, of which 52 are positive cases.

We apply BLR and MarBLR for online model recalibration, logistic model revision, and online ensembling, using the same procedures described for the simulations. Logistic model revision was restricted to three predictors based on clinical knowledge: age, history of COPD, and history of smoking. In the online ensembling experiment, we refit the gradient boosted tree on all prior data every 270 observations (across 400 time points). We selected the hyperparameters for BLR and MarBLR such that the regret bound was no more than 5% of the

initially estimated loss for the locked model. For computational speed, we run BLR and MarBLR on batches of $n = 10$ observations.

We split the data into four time periods with an equal number of observations. We evaluate the deployed models using walk-forward testing, i.e. the forecasted probability was compared to the observed outcome at each time point, and estimated the AUC, ECI, and NLL for each time period. We average these performance metrics across the time periods to calculate aAUC, aECI, and aNLL. Confidence intervals are constructed using 200 bootstrap replicates.

All empirical studies were performed on an Intel Gold 6240 CPU. The computation time for each experiment was no more than thirty minutes, except for the online ensembling procedure for the COPD data analysis. This particular experiment took around three hours, where the bulk of the time was spent on refitting a GBT each time new data was collected.

RESULTS

**Theoretical Results**

We derived Type I and II regret bounds for BLR and MarBLR, which give us the theoretical guarantees regarding the safety and effectiveness of the two procedures in the presence of distribution shifts. The regret bounds provide guidance for choosing between BLR and MarBLR as well as their hyperparameters. The results are finite-sample, do not assume that the Bayesian modeling assumptions are correct, and hold even if the data is adversarially chosen. A summary of the regret bounds is shown in Table 1. Theorems and proofs are provided in the Appendix. Below, we highlight the trade-off between Type I and II regret as we vary hyperparameters in the two procedures.

Table 1 Overview of theoretical results. The regret bounds are displayed in asymptotic notation. $\theta^*_{\tau_{locked}}$ is the best locked model revision sequence in retrospect and $\theta^*_\tau$ is the best dynamic model revision sequence with revision times $\tau$ in retrospect. Symbol meanings: d=dimension of logistic revision parameters, T=total number of time steps, $\theta_{init}$=initial logistic parameters, $\delta^2$=inflation factor for MarBLR posterior, $\alpha$=update probability in MarBLR, $p_0$=MarBLR prior over update times.

|  | Type I regret bound | Type II τ-regret bound |
| --- | --- | --- |
| BLR | $\frac{d}{T} \log\left(1 + \frac{T}{d}\right)$ | $\frac{d}{T} \log\left(1 + \frac{T}{d}\right)$ <br> $+ \frac{1}{T} \left\|\theta^*_{\tau_{\text{locked}}} - \theta_{\text{init}}\right\|_2^2 + \frac{1}{T} \sum_{j=1}^{|\tau|} (\tau_{j+1} - \tau_j) \left\|\theta^*_{\tau_{\text{locked}}} - \theta^*_{\tau_j}\right\|^2$ |
| MarBLR | $\frac{d}{T} \log\left(1 + \frac{T}{d}\right) + d\alpha \log\left(1 + \frac{\delta^2 T}{d}\right)$ | $\sum_{j=2}^{|\tau|} \frac{d}{T} \log\left(1 + \frac{1}{\delta^2} + \frac{\tau_j - \tau_{j-1}}{d}\right)$ <br> $+ \frac{1}{T} \left\|\theta^*_1 - \theta_{\text{init}}\right\|_2^2 + \frac{1}{T} \sum_{j=2}^{|\tau|} \frac{1}{\delta^2} \left\|\theta^*_{\tau_j} - \theta^*_{\tau_{j-1}}\right\|_2^2$ <br> $- \frac{1}{T} \log p_0(\tau) + \frac{|\tau| - 1}{T} d \log \delta$ |

Type I regret for BLR converges at the rate of $O(d/T \log(T/d))$, where $d$ denotes the dimension of the logistic revision parameter, and T denotes the total number of time steps.

This is best currently-known rate for online logistic regression [20]. Because this regret bound converges quickly to zero as T increases, it can be used to meaningfully control Type I regret on realistic time horizons. Although not shown in Table 1, the Type I regret for BLR decreases to zero as we shrink the prior covariance matrix $\Sigma_{init}$ and concentrate the prior around locking the original model (and its revision).

In comparison, the Type I regret bound for MarBLR includes an additional error term of $O(d\alpha \log(\delta^2 T))$ that corresponds to the prior probability of revision sequences that change over time. As we increase $\alpha$ and $\delta^2$, MarBLR prior puts less weight on the original model and searches over more dynamic revision sequences. Because this additional error term does not go to zero as $T$ increases, so one needs to choose the MarBLR hyperparameters with care. This illustrates how Type I regret increases when we try to estimate more complex model revision sequences using MarBLR.

Type II regret for MarBLR is $O(d/T \log(T/d))$ plus the distance between the dynamic oracle and the prior. If the true sequence of revision update times τ is known, selecting $\alpha = |\tau|/T$ in MarBLR minimizes the bound on the average expected loss. On the other hand, BLR assumes the oracle model revision sequence is static and sets $\alpha = 0$. Consequently, its Type II regret bound incurs an additional term that quantifies the error in approximating the dynamic oracle $\theta_\tau^*$ with a locked oracle $\theta_{\tau_{locked}}^*$.

To summarize, BLR has smaller Type I but larger Type II regret bounds than MarBLR for the same Gaussian prior at time $t = 1$. We can further fine-tune the Type I regret control achieved by the two methods by selecting more or less diffuse priors.

**Simulation Studies**

*Scenario 1: Online recalibration of a locked underlying model*
By design, the online recalibration methods aim to minimize the aNLL with respect to the general population. Indeed, we find that all the online recalibration methods significantly improved aNLL and aECI compared to the locked model, with BLR and MarBLR achieving the smallest values (Table 2, Figure 3). Online recalibration also improved aNLL and aECI within each subpopulation, but different methods achieved different levels of calibration of across the subgroups. In general, we find that that subgroup-aware recalibration leads to more similar aECI between the subpopulations than univariate recalibration. For instance, MarBLR achieved aECIs of 0.590 (95% CI 0.543-0.637) and 0.176 (95% CI 0.158- 0.193) in subgroups A and B, respectively, when performing univariate recalibration. In contrast, the aECIs for subgroups A and B are 0.498 (95% CI 0.429-0.566) and 0.233 (95% CI 0.208-0.258), respectively, when performing subgroup-aware recalibration. Finally, we note that there is a spike in the ECI at early time points but quickly disappears as data accumulates.

*Table 2 – Average performance of online logistic recalibration methods of a fixed underlying prediction model and simulated patients subgroups A and B with prevalence 20% and 80%, respectively (Scenario 1). Methods include Bayesian logistic revision (BLR) and its Markov variant MarBLR, Adam, cumulative refitting of a logistic regression model (CumulativeLR), and locking the original model (Locked). Standard errors over 50 replicates are shown in parentheses. Abbreviati ons: aECI = average estimated calibration index, aNLL = average negative log likelihood.*

|  | Average ECI (aECI) | | | Average NLL (aNLL) | | |
| --- | --- | --- | --- | --- | --- | --- |
| Subgroup | A | B | Combined | A | B | Combined |

| | | | | | | |
|---|---|---|---|---|---|---|
| Locked | 1.371 (0.009) | 0.612 (0.005) | 0.910 (0.005) | 0.661 (0.001) | 0.559 (0.000) | 0.580 (0.000) |
| *Online univariate logistic recalibration* | | | | | | |
| MarBLR | 0.590 (0.024) | 0.176 (0.009) | 0.301 (0.016) | 0.622 (0.001) | 0.538 (0.000) | 0.555 (0.001) |
| BLR | 0.655 (0.027) | 0.199 (0.012) | 0.346 (0.019) | 0.624 (0.001) | 0.539 (0.001) | 0.556 (0.001) |
| Adam | 0.878 (0.018) | 0.273 (0.012) | 0.494 (0.014) | 0.635 (0.001) | 0.543 (0.001) | 0.561 (0.001) |
| CumulativeLR | 0.745 (0.025) | 0.225 (0.013) | 0.403 (0.019) | 0.629 (0.001) | 0.541 (0.001) | 0.558 (0.001) |
| *Online subgroup-aware logistic recalibration* | | | | | | |
| MarBLR | 0.498 (0.035) | 0.233 (0.013) | 0.241 (0.013) | 0.616 (0.001) | 0.541 (0.001) | 0.556 (0.001) |
| BLR | 0.465 (0.037) | 0.270 (0.016) | 0.265 (0.018) | 0.615 (0.002) | 0.542 (0.001) | 0.557 (0.001) |
| Adam | 0.439 (0.025) | 0.426 (0.011) | 0.320 (0.016) | 0.614 (0.001) | 0.550 (0.001) | 0.562 (0.001) |
| CumulativeLR | 0.449 (0.040) | 0.303 (0.017) | 0.311 (0.020) | 0.616 (0.002) | 0.544 (0.001) | 0.558 (0.001) |

*Scenario 2: Online logistic revision of a locked underlying model*

BLR and MarBLR learned beneficial logistic revisions faster than the other online methods across all data settings (Table 3, Figure 4), improving in both model discrimination and calibration over the locked model. BLR and MarBLR significantly improved model discrimination in the setting with IID data after an initial shift. We also observed improvements in model discrimination when there were cyclical distribution shifts, but to a lesser extent. In the Decay data setting, the locked model had an initial AUC of 0.85 and achieved an average AUC of 0.803 (95% CI 0.801-0.805). In contrast, BLR and MarBLR slowed down the performance decay, achieving an aAUC of 0.819 (95% CI 0.817-0.821).

*Table 3 Performance of online logistic revision of fixed underlying prediction model with respect to the forecasted score and patient variables (Scenario 2). Simulated data settings include IID data after an initial shift (Initial Shift), nonstationary data that cycles between three distributions (Cyclical), nonstationary data where the performance of the original model decays over time (Decay). Methods include Bayesian logistic revision (BLR) and its Markov variant MarBLR, Adam, cumulative refitting of a logistic regression model (CumulativeLR) and locking the original model (Locked). Standard errors over 50 replicates are shown in parentheses. Abbreviations: aAUC = area under the receiver operating characteristic curve, aECI = average estimated calibration index, aNLL = average negative log likelihood.*

| | aAUC | aECI | aNLL |
|---|---|---|---|
| *Initial Shift* | | | |
| MarBLR | 0.799 (0.001) | 0.288 (0.012) | 0.554 (0.001) |
| BLR | 0.800 (0.001) | 0.278 (0.013) | 0.553 (0.001) |
| Adam | 0.795 (0.001) | 0.576 (0.014) | 0.574 (0.001) |
| CumulativeLR | 0.797 (0.001) | 0.417 (0.017) | 0.563 (0.001) |
| Locked | 0.767 (0.001) | 1.637 (0.008) | 0.661 (0.001) |
| *Cyclical* | | | |
| MarBLR | 0.834 (0.001) | 0.211 (0.011) | 0.510 (0.001) |
| BLR | 0.834 (0.001) | 0.202 (0.012) | 0.509 (0.001) |

| | | | |
|---|---|---|---|
| Adam | 0.834 (0.001) | 0.293 (0.011) | 0.514 (0.001) |
| CumulativeLR | 0.834 (0.001) | 0.240 (0.013) | 0.511 (0.001) |
| Locked | 0.826 (0.001) | 0.605 (0.002) | 0.542 (0.000) |
| *Decay* | | | |
| MarBLR | 0.819 (0.001) | 0.281 (0.015) | 0.532 (0.001) |
| BLR | 0.819 (0.001) | 0.277 (0.017) | 0.532 (0.001) |
| Adam | 0.820 (0.001) | 0.379 (0.015) | 0.537 (0.001) |
| CumulativeLR | 0.819 (0.001) | 0.337 (0.018) | 0.535 (0.001) |
| Locked | 0.803 (0.001) | 0.979 (0.002) | 0.588 (0.000) |

*Scenario 3: Online ensembling of a locked and continuously refitted black-box model*
CumulativeLR, BLR, and MarBLR achieved the top performance across the different model refitting procedures and data distributions (Table 4, Figure 5). In All-Refit, the evolving model had better performance than the original model so the online revisers learned to place more weight on the evolving model, thereby improving model calibration and discrimination. In Subset-Refit, we simulated an unreliable evolving black-box model to test how the online revisers respond to sudden model deterioration. All the online model revisers drop in performance when the evolving model suddenly deteriorates, but they recover over time, some faster than others.

We can gain more insight into the operating characteristics of BLR and MarBLR by visualizing how the logistic revision coefficients change over time (Appendix Figure A.1). For All-Refit, the two methods gradually increased the importance of the evolving model and decreased the importance of the locked model. By the end of the time period, both methods assigned higher importance to the evolving model. As expected, this switch in model importance occurs earlier in the nonstationary setting. For Subset-Refit, BLR and MarBLR increase the coefficient of the evolving model during the time period when it was a good predictor of the outcome (t < 100). Once the evolving model decayed in prediction accuracy, its coefficient starts to decrease towards zero.

*Table 4 Performance of online logistic revision as a wrapper for a continually refitted gradient boosted tree (Scenario 3). Refitting procedures include refitting on all available data (All-Refit) and refitting on the most recent window of data (Subset-Refit). Methods include Bayesian logistic revision (BLR) and its Markov variant MarBLR, Adam, cumulative refitting of a logistic regression model (CumulativeLR), and locking the original model (Locked). Standard errors over 50 replicates are shown in parentheses. Abbreviati ons: aAUC = area under the receiver operating characteristic curve, aECI = average estimated calibration index, aNLL = average negative log likelihood.*

| | aAUC | aECI | aNLL |
|---|---|---|---|
| *Initial Shift, All-Refit* | | | |
| MarBLR | 0.689 (0.002) | 0.452 (0.024) | 0.646 (0.001) |
| BLR | 0.689 (0.002) | 0.482 (0.027) | 0.647 (0.001) |
| Adam | 0.621 (0.001) | 1.085 (0.029) | 0.702 (0.001) |
| CumulativeLR | 0.690 (0.002) | 0.448 (0.021) | 0.646 (0.001) |
| Locked | 0.624 (0.001) | 2.695 (0.017) | 0.762 (0.001) |
| *Initial Shift, Subset-Refit* | | | |
| MarBLR | 0.658 (0.001) | 0.707 (0.033) | 0.671 (0.001) |
| BLR | 0.658 (0.001) | 0.696 (0.035) | 0.671 (0.001) |
| Adam | 0.643 (0.001) | 1.128 (0.035) | 0.695 (0.001) |

| | | | |
|---|---|---|---|
| CumulativeLR | 0.659 (0.001) | 0.787 (0.040) | 0.675 (0.001) |
| Locked | 0.624 (0.001) | 2.695 (0.017) | 0.762 (0.001) |
| *Decay, All-Refit* | | | |
| MarBLR | 0.726 (0.001) | 0.238 (0.014) | 0.617 (0.001) |
| BLR | 0.725 (0.001) | 0.224 (0.014) | 0.617 (0.001) |
| Adam | 0.689 (0.001) | 0.514 (0.019) | 0.652 (0.001) |
| CumulativeLR | 0.727 (0.001) | 0.211 (0.014) | 0.616 (0.001) |
| Locked | 0.688 (0.001) | 1.218 (0.007) | 0.679 (0.000) |
| *Decay, Subset-Refit* | | | |
| MarBLR | 0.703 (0.001) | 0.456 (0.029) | 0.639 (0.001) |
| BLR | 0.704 (0.001) | 0.450 (0.031) | 0.639 (0.001) |
| Adam | 0.698 (0.001) | 0.647 (0.026) | 0.650 (0.001) |
| CumulativeLR | 0.705 (0.001) | 0.533 (0.034) | 0.642 (0.001) |
| Locked | 0.688 (0.001) | 1.218 (0.007) | 0.679 (0.000) |

**COPD case study**

Due to real-world temporal shifts, the original model decayed in calibration (Table 5, Figure 6). Online logistic recalibration gradually decreased the forecasted score, which reflects the general downward trend in COPD diagnosis rates. It improved the aECI from 1.526 (95% CI 1.371-1.680) to 0.693 (95% CI 0.624-0.762) with BLR and 0.450 (95% CI 0.411-0.490) with MarBLR. The performance of online logistic revision was similar to online logistic recalibration. We observe significant improvements when BLR and MarBLR ensemble the original and continually refitted GBTs. Using MarBLR, we achieve an aAUC of 0.925 (95% CI 0.914-0.935) compared to the locked model's aAUC of 0.904 (95% CI 0.892-0.916). Although the initial coefficient for the refitted model is set to zero, BLR and MarBLR increased the weight of the refitted model and decreased that of the original model. Finally, BLR and MarBLR had similar AUC as the continually refitted model but were better calibrated. Perhaps even more importantly, BLR and MarBLR offer safety guarantees whereas the continual refitting procedure does not.

*Table 5 -- Results from COPD risk prediction task using Bayesian logistic revision (BLR) and its Markov variant MarBLR. 95% bootstrap confidence intervals are shown in parentheses. Abbreviations: aAUC = average area under the receiver operating characteristic curve, aECI = average estimated calibration index, aNLL = average negative log likelihood.*

| | aAUC | aECI | aNLL |
|---|---|---|---|
| Locked | 0.904 (0.892,0.916) | 1.526 (1.371,1.680) | 0.099 (0.097,0.101) |
| *Online logistic recalibration* | | | |
| MarBLR | 0.907 (0.895,0.919) | 0.450 (0.411,0.490) | 0.079 (0.077,0.082) |
| BLR | 0.906 (0.894,0.918) | 0.693 (0.624,0.762) | 0.082 (0.080,0.085) |
| *Online logistic revision* | | | |
| MarBLR | 0.909 (0.897,0.921) | 0.460 (0.427,0.493) | 0.078 (0.075,0.080) |
| BLR | 0.907 (0.894,0.919) | 0.708 (0.647,0.769) | 0.082 (0.079,0.084) |
| *Online ensembling of the original and continually refitted models* | | | |
| MarBLR | 0.925 (0.914,0.935) | 0.482 (0.450,0.514) | 0.072 (0.070,0.074) |
| BLR | 0.924 (0.913,0.935) | 0.429 (0.401,0.458) | 0.071 (0.069,0.073) |
| Continual Refit Only | 0.924 (0.914, 0.935) | 0.657 (0.624, 0.690) | 0.078 (0.076, 0.080) |

## DISCUSSION

Performance degradation of clinical ML algorithms occur for a variety of reasons, such as abrupt system-wide changes in the record keeping system [27], changes in the event rate and patient case mix [28], and changes in clinical order patterns [29]. A growing number of papers have highlighted the need for regular monitoring and updating of clinical prediction algorithms [3,30], but there are currently no online procedures with practical performance guarantees. Instead, much of the literature has focused on learning one-time model updates [31–33], which have inflated error rates when applied repeatedly over time. Dynamic model updating has also been suggested for clinical prediction models [14,34,35] but theoretical guarantees in the presence of model misspecification and dataset shifts have been lacking.

In this work, we show that online model revision by BLR or MarBLR is a promising solution that both exhibits strong empirical performance and provides theoretical guarantees. We derived Type I regret bounds that guarantee the online revision methods will be noninferior to locking the original model (on average) in the presence of arbitrary distribution shifts and Type II regret bounds that guarantee the methods quickly learn beneficial revisions. Our data analyses verified these results: BLR and MarBLR consistently outperformed locking the original model by slowing down performance decay in nonstationary settings and gradually improved overall performance in stationary settings. In contrast, other online methods did not provide the same theoretical guarantees and tended to learn good model revisions more slowly.

The key difference between BLR and MarBLR is that the former assumes the oracle sequence of model revisions is static, whereas the latter allows for dynamic sequences. The theoretical results highlight a tradeoff between the two procedures: BLR incurs higher bias because of its simplifying assumptions, whereas MarBLR is more sensitive to sampling noise because it searches over a richer class of model revision sequences. In practice, we find that BLR is highly competitive with MarBLR. Although MarBLR tended to outperform BLR in settings with more severe distribution shifts, the performance differences were negligible in most, if not all, cases. Given that BLR is a simpler procedure involving fewer hyperparameters, we believe BLR is sufficient in most settings. Moreover, BLR can be implemented using standard software packages for Bayesian inference [36,37].

We highlight that MarBLR and BLR only provide theoretical guarantees in terms of the average negative log likelihood. As seen in the simulations, MarBLR and BLR may perform worse than the original model for a few time points because of small sample sizes early on, or because the continual refitting procedure for the underlying prediction model suddenly introduces a bad update. This serves as a point of caution. Many online learning methods, including ours, do not guarantee that the deployed model will outperform the static model at all time points. Nevertheless, MarBLR and BLR are guaranteed to recover from sudden performance decay such that the *average* performance compares favorably to locking the original model. Analogously, MarBLR and BLR are not guaranteed to outperform the original model within a particular patient subgroup if the online revision procedure trains on data from the general population. Although the results from the first simulation are promising, we can only control Type I regret if separate instances of MarBLR/BLR are deployed within each patient subgroup. Future work should evaluate the algorithmic fairness of MarBLR and BLR in a wide variety of

settings and introduce any necessary extensions to ensure fairness will be maintained over time. This paper has only explored the simplest setting with two patient subgroups defined a priori; in practice, one may be interested in many more subgroups and may even want to preserve fairness across subgroups that have yet to be defined clinically.

A key assumption in this work is that one can measure gold standard outcomes and perform an unbiased evaluation of the underlying prediction model. One must verify these assumptions hold, at least approximately, before applying MarBLR and BLR. In our case study, we considered a risk prediction model for predicting COPD diagnosis. This model may influence clinical decision making and perturb the distribution of both the covariates and the outcome, thereby complicating model evaluation and updating. This issue called "performative prediction" [38] can be even more severe in high acuity settings where ML recommendations are designed to change clinical workflows. Nevertheless, one may be able to obtain more accurate patient outcomes by incorporating a lag time, during which missed diagnoses are corrected. Another option is to consider the counterfactual framework and define the true outcome as the outcome that would have been observed, had the ML prediction not been made available to the clinician. Two avenues for identifying this true outcome are to either randomize a subset of patients to receive no ML prediction or use a causal inference/missing data framework. Both options come with substantial challenges and require further research [39].

In this paper, we require the user to decide the class of logistic revisions upfront. That is, they must decide between online model recalibration versus revision; if they choose the latter, they must decide which variables to include and whether to incorporate a continually refitted model. Because future data distributions are unknown, it is difficult to anticipate which type of model revision will lead to best online performance. As seen in the COPD case study, the logistic model revisions with respect to the selected patient variables did not improve model discrimination but incorporating a continually refitted model did. Instead of requiring a model revision class to be selected upfront, future extensions of MarBLR/BLR may be able to incorporate model selection using a hierarchical modeling approach such as in [14]. In addition, new biomarkers will be discovered as the biomedical field continues to advance. An important use case is to let the model developer add these newly discovered biomarkers to the online model revision procedure.

Finally, a limitation of the current work is that the derived regret bounds scale linearly with the number of variables used during online model revision. As such, MarBLR and BLR may not provide meaningful safety guarantees in high-dimensional settings. Because many modern ML algorithms analyze a large number of variables, future work should look to refine our regret bounds by characterizing, say, the L1-norm or sparsity of the oracle model revision [20].

## CONCLUSION

Our theoretical and empirical results support the use of online model revision by BLR or MarBLR over other online methods for regular monitoring and updating of clinical prediction algorithms when performance drift is of concern.



\begin{figure}[H]
	\centering
	\includegraphics[width=0.5\linewidth]{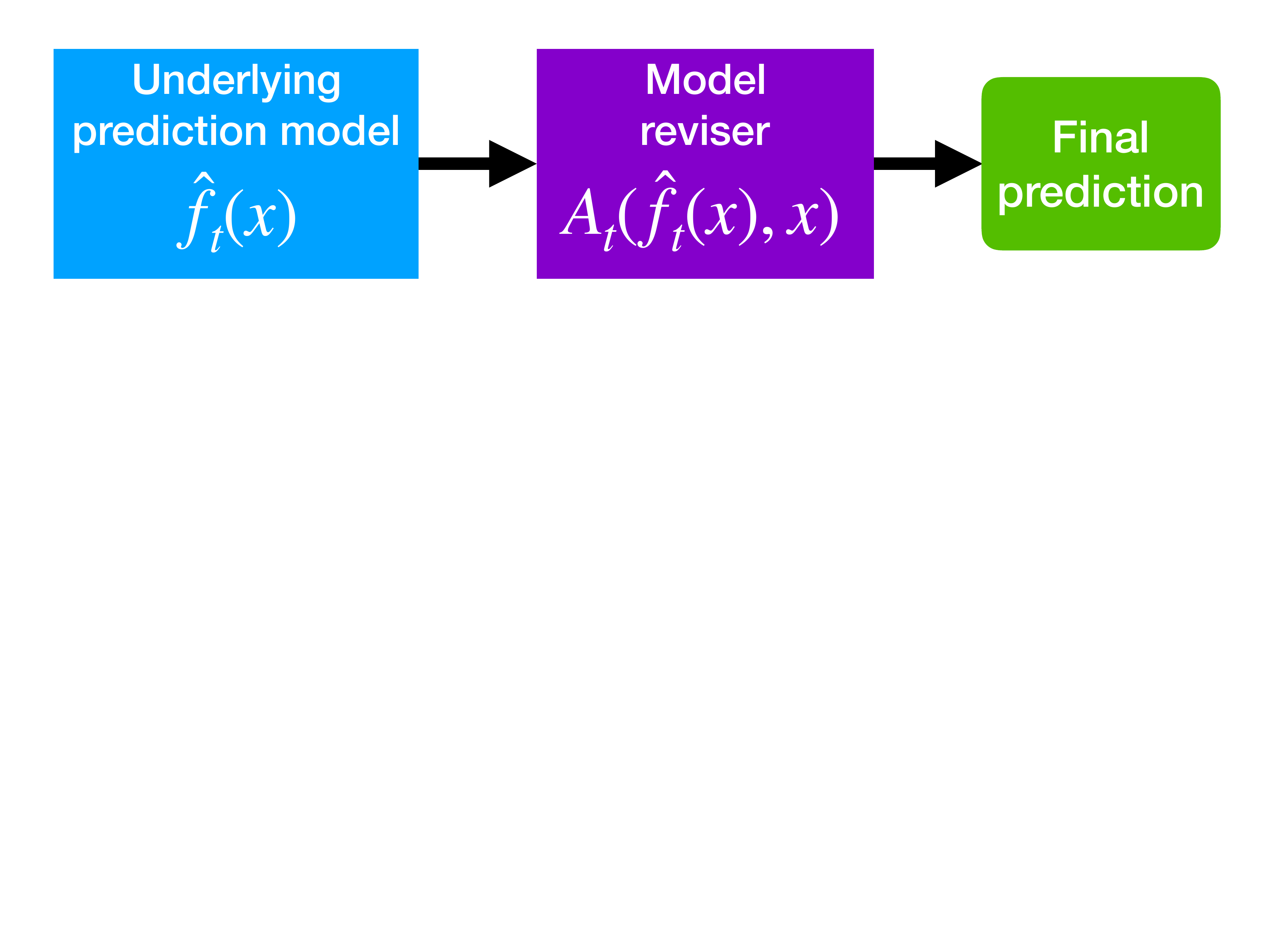}
	\vspace{-1.6in}
	\caption{
	Given a patient with variables $x$, the model reviser $\hat{A}_t$ wraps around an underlying ML model $\hat{f}_t$ to predict the true probability of having or developing a disease (or outcome).
	The focus of this work is the design of an online model reviser.
	}
	\label{fig:overview}
\end{figure}

\begin{figure}[H]
	\centering
	\begin{subfigure}{0.45\linewidth}
		\centering
	\includegraphics[width=\linewidth]{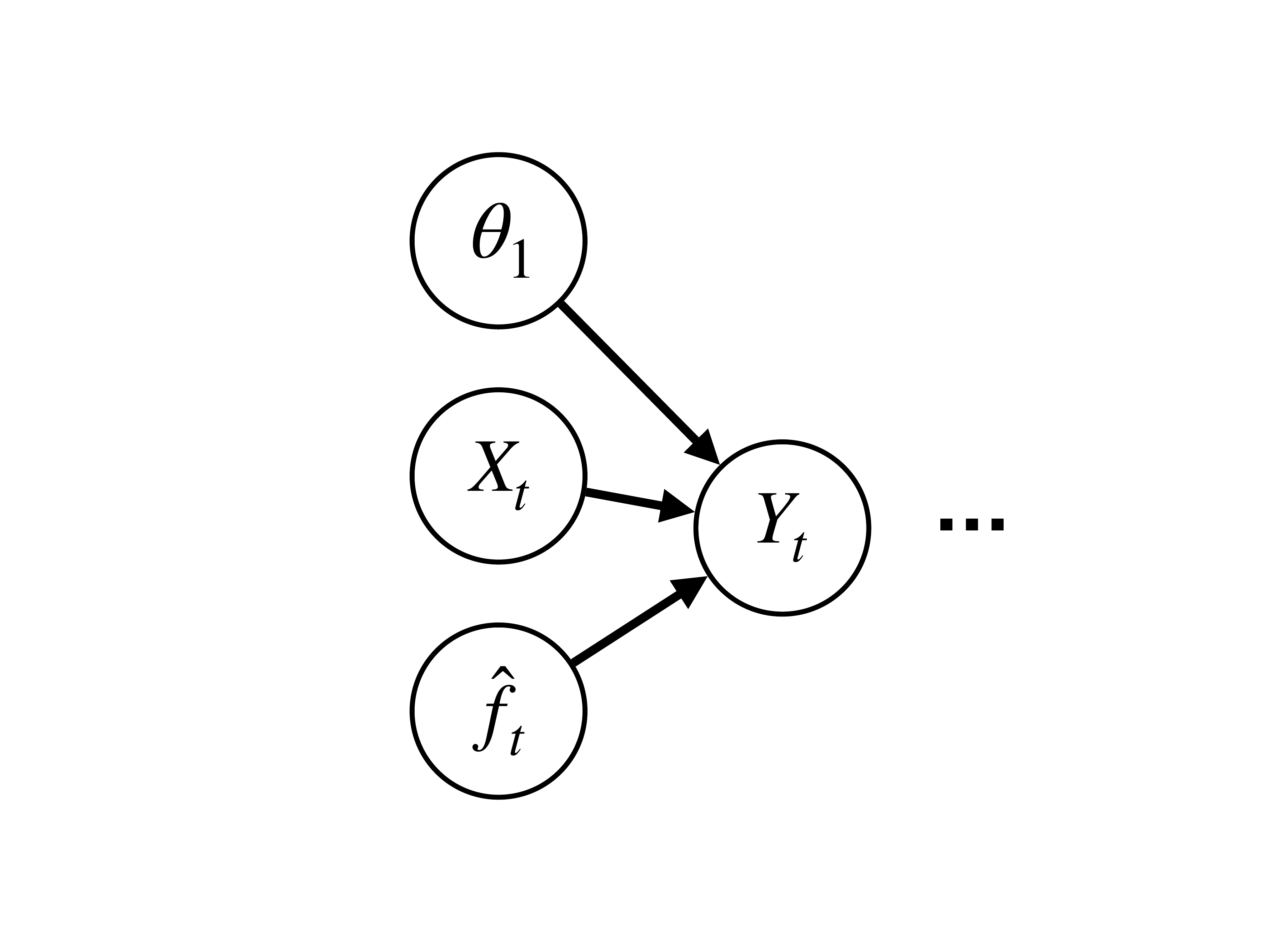}
	\caption{Bayesian logistic revision (BLR) estimates the model revision parameters $\theta_t$ for the underlying prediction model $\hat{f}_t$ with the simplifying assumption that the data $(X_t, Y_t)$ are independently and identically distributed for a constant set of model revision parameters over all time points $t = 1,...,T$.}
	\end{subfigure}
\hspace{1cm}
	\begin{subfigure}{0.45\linewidth}
		\centering
	\includegraphics[width=\linewidth]{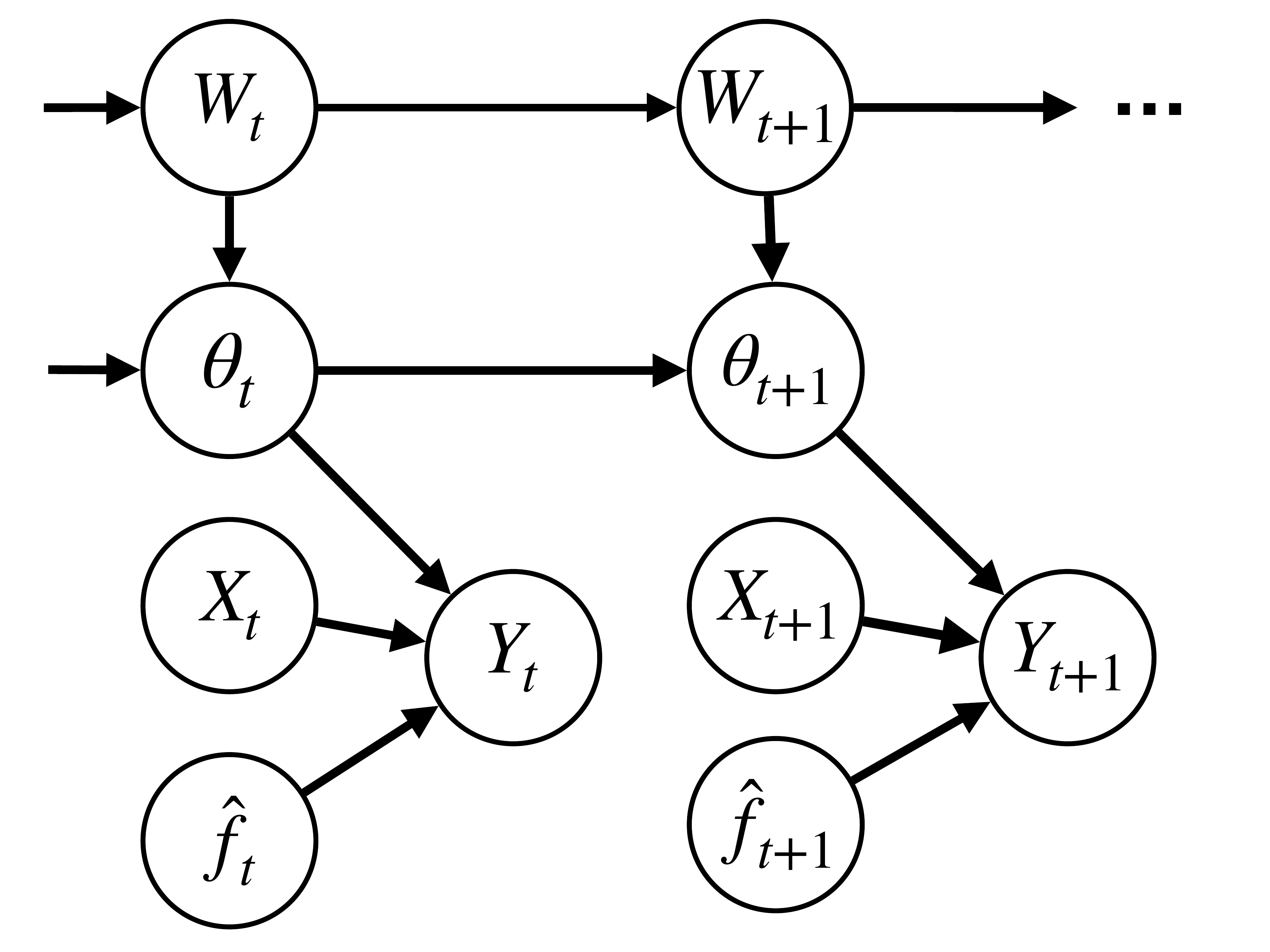}
	\caption{
		MarBLR defines a prior over revision parameter sequences that change over time. It assumes that the revision parameters change with probability $\alpha$ at each time $t$, as modeled by a binary random variable $W_t$. It supposes that changes in the model revision parameters follow a Gaussian prior centered at zero.
	}
	\end{subfigure}
	\caption{
	BLR and its Markov variant MarBLR update the deployed model revision at each time point per the evolving Bayesian posterior.
	Theoretical guarantees for BLR and MarBLR hold under misspecification of the logistic model and/or priors.
	}
	\label{fig:dynami}
\end{figure}

\begin{figure}[H]
	\centering
	\begin{subfigure}{\linewidth}
		\hspace{-0.2cm}
		\includegraphics[width=0.22\linewidth, trim=0 0.3cm 0 0]{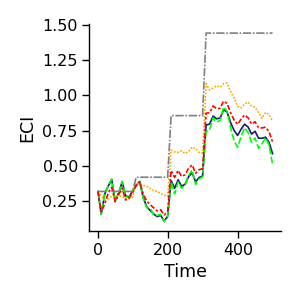}
		\hspace{-0.3cm}
		\includegraphics[width=0.66\linewidth]{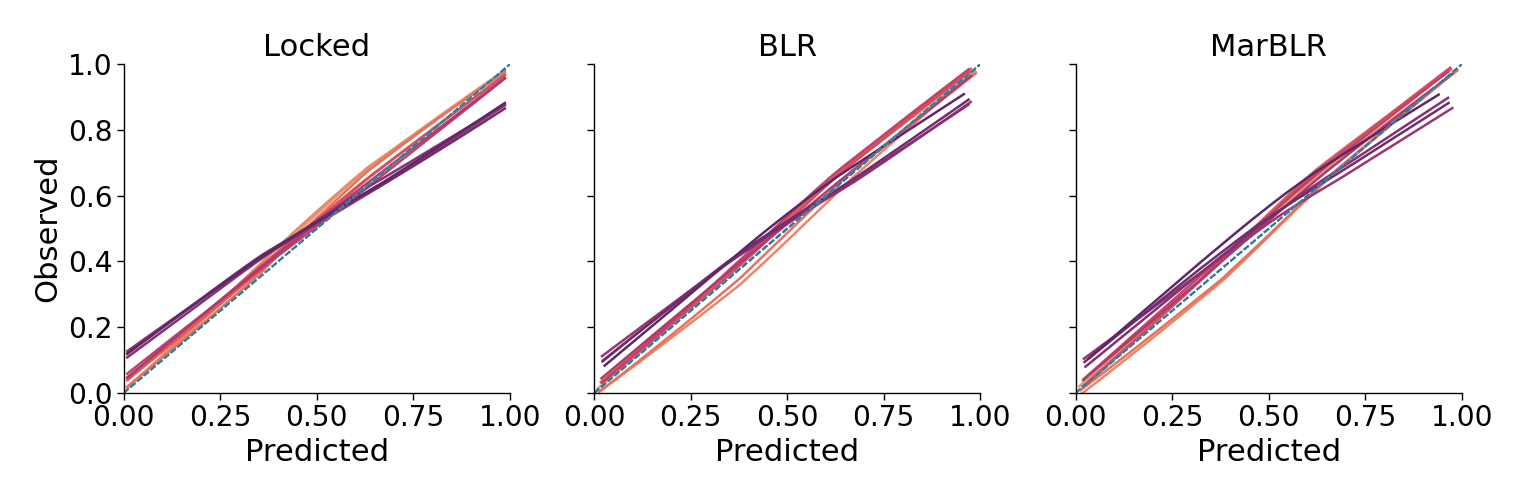}
		\includegraphics[width=0.13\linewidth]{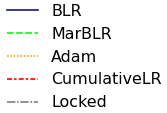}
		\caption{Univariate recalibration. Calibration measured with respect to the general population.}
		\label{fig:recalib_general_shift}
	\end{subfigure}
	\begin{subfigure}{\linewidth}
		\includegraphics[width=0.22\linewidth,valign=t]{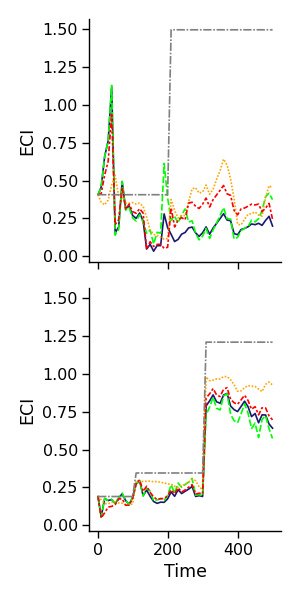}
		\hspace{-0.4cm}
		\includegraphics[width=0.66\linewidth,valign=t]{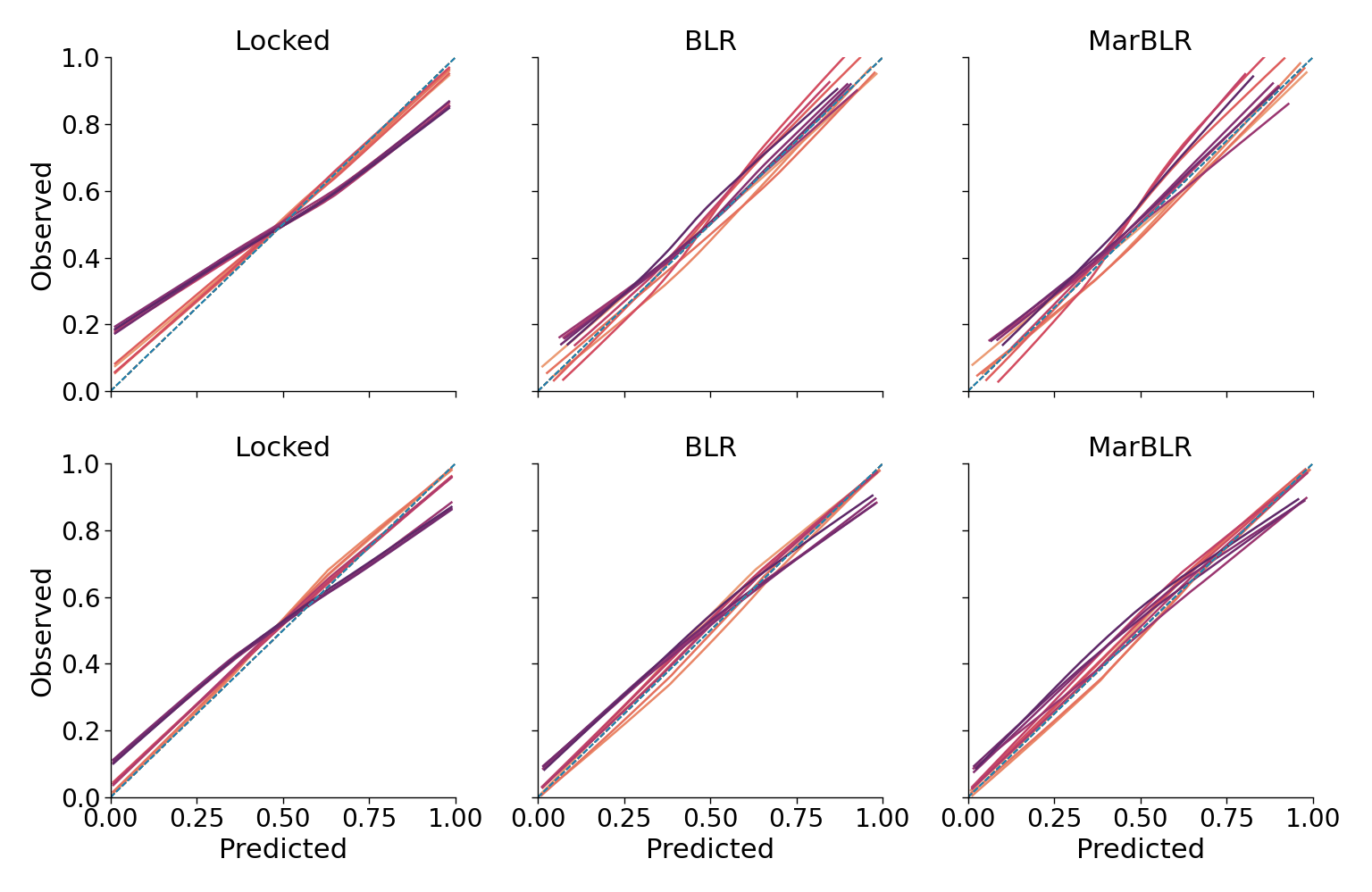}
		\hspace{-0.2cm}
		\includegraphics[width=0.07\linewidth,valign=t]{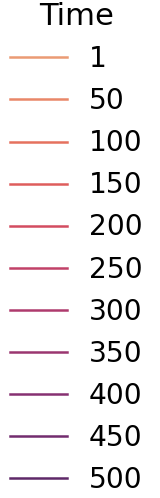}
		\caption{Subgroup-aware recalibration. Calibration measured with respect to subgroups A (top) and B (bottom).}
		\label{fig:two_pops_shift}
	\end{subfigure}
\caption{
Results from online model recalibration of a fixed underlying prediction model in a patient population with patient subgroups A and B with prevalence 20\% and 80\% (Scenario 1).
Left: Estimated calibration index (ECI) at each time point. Right: Calibration curves for the original model and the revised versions from BLR and MarBLR. The ideal calibration curve is the identity function, which has an ECI of zero.
}
	\label{fig:two_pops}
\end{figure}

\begin{figure}[H]
	\centering
		\hspace{-2.2cm}\includegraphics[width=0.5\linewidth]{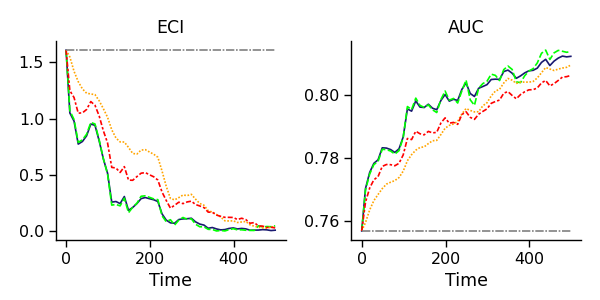}\\
		\hspace{-2.2cm}\includegraphics[width=0.5\linewidth]{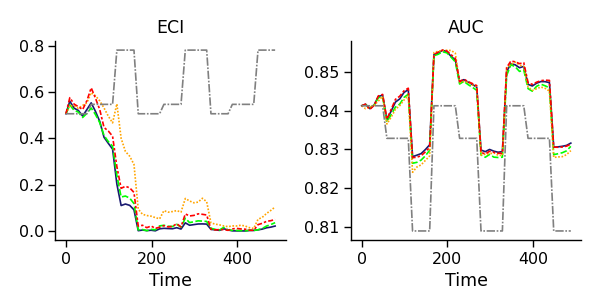}\\
		\includegraphics[width=0.5\linewidth]{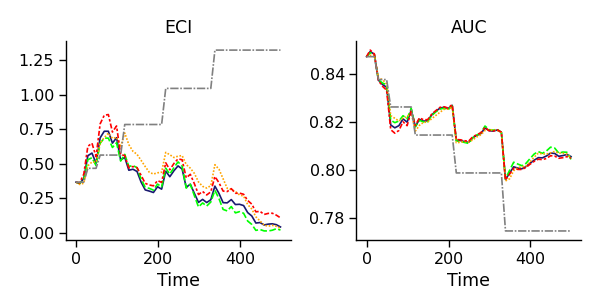}\includegraphics[width=0.15\linewidth]{legend_methods}
\caption{Results from online logistic revision of a fixed underlying model with respect to the forecasted score and ten patient variables (Scenario 2), in terms of the estimated calibration index (ECI, left) and AUC (right).
Data is simulated to be stationary over time after an initial shift (\texttt{Initial Shift}, top), shift in a cyclical fashion (\texttt{Cyclical}, middle), and shift such that the original model decays in performance over time (\texttt{Decay}, bottom). 
All online logistic revision methods outperformed locking the original model in terms of the average ECI and AUC, with BLR and MarBLR performing the best.
Note that the revised models were worse than the original model briefly in the \texttt{Decay} setting.
}
	\label{fig:dynamic_sim}
\end{figure}

\begin{figure}[H]
	\centering
	\begin{subfigure}{0.5\linewidth}
	\includegraphics[width=0.9\linewidth]{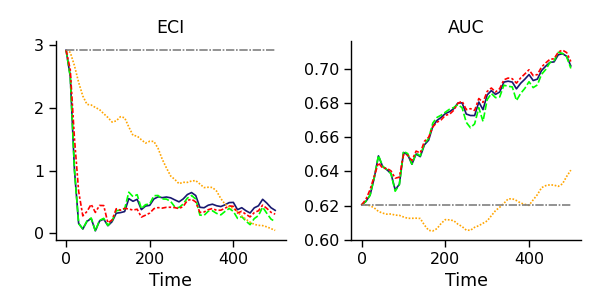}
	\caption{\texttt{Initial Shift}, \texttt{All-Refit}}
	\label{fig:iid_good}
	\end{subfigure}
\hspace{-0.5cm}
	\begin{subfigure}{0.5\linewidth}
	\includegraphics[width=0.8\linewidth]{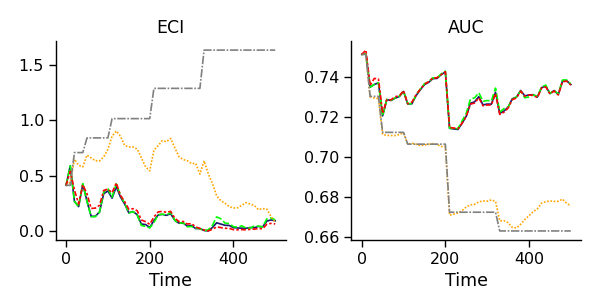}
	\caption{\texttt{Decay}, \texttt{All-Refit}}
	\label{fig:shift_good}
	\end{subfigure}

%	\hspace{-1cm}
	\begin{subfigure}{0.5\linewidth}
	\includegraphics[width=0.9\linewidth]{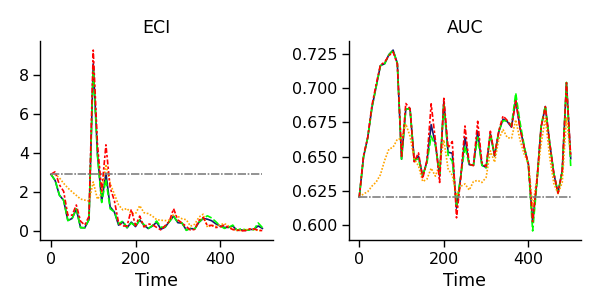}
	\caption{\texttt{Initial Shift}, \texttt{Subset-Refit}}
	\label{fig:iid_bad}
	\end{subfigure}
\hspace{-0.5cm}
	\begin{subfigure}{0.5\linewidth}
	\includegraphics[width=0.8\linewidth]{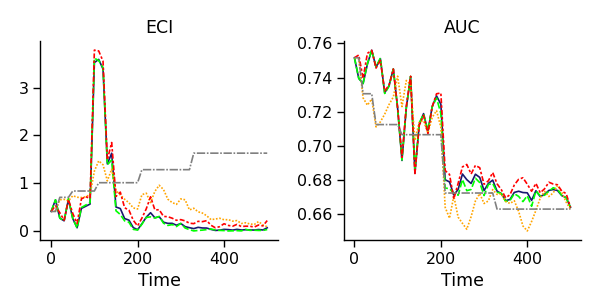}
	\includegraphics[width=0.19\linewidth]{legend_methods}
	\caption{\texttt{Decay}, \texttt{Subset-Refit}}
	\label{fig:shift_bad}
	\end{subfigure}
\caption{Model calibration and discrimination (left and right panels, respectively) from online ensembling of the original model with an underlying prediction model that is continually refitted over time (Scenario 3).
Data is simulated to be stationary over time after an initial shift (\texttt{Initial Shift}) and nonstationary such that the original model decays in performance over time (\texttt{Decay}). 
Underlying prediction model is updated by continually refitting on all previous data (\texttt{All-Refit}) or refit on the most recent subset of data (\texttt{Subset-Refit}).
Note that \texttt{Subset-Refit} simulates a sudden drop in performance for the continually refitted model at time $t = 100$ and, consequently, across all online ensembling procedures.
BLR and MarBLR recover from this sudden performance decay and achieve better performance than locking the original model in terms of the average ECI and AUC.
}
\label{fig:refitting}
\end{figure}

\begin{figure}[H]
	\centering
	\begin{subfigure}{\linewidth}
		\centering
		\includegraphics[width=0.78\linewidth]{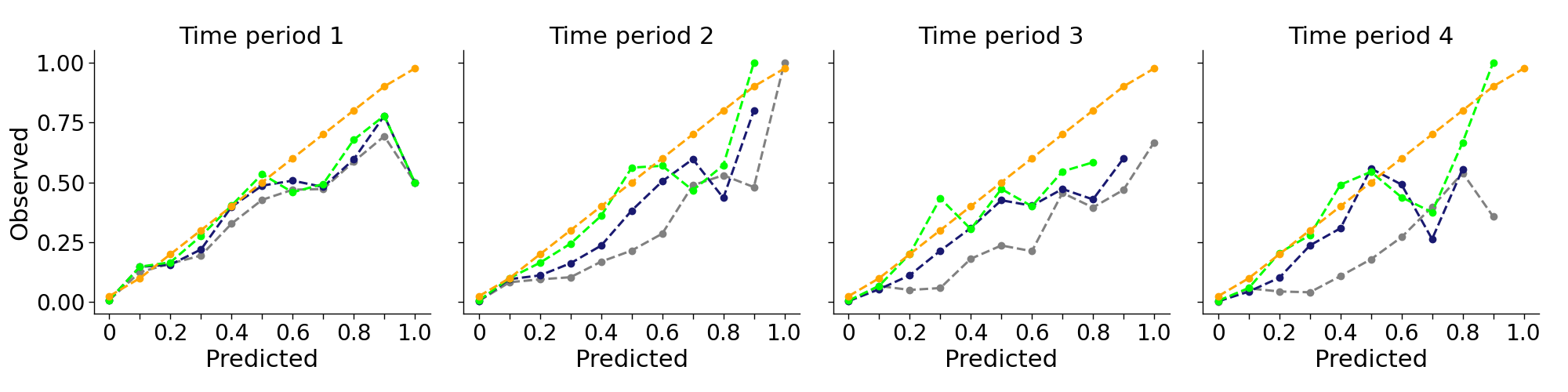}
		\hspace{-0.2cm}
		\includegraphics[width=0.1\linewidth]{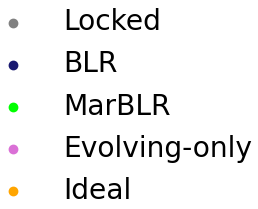}
	\caption{Online recalibration of a fixed prediction model}
	\label{fig:online_recalib_copd}
	\end{subfigure}	
	\begin{subfigure}{\linewidth}
		\centering
	\hspace{-1.8cm}
		\includegraphics[width=0.78\linewidth]{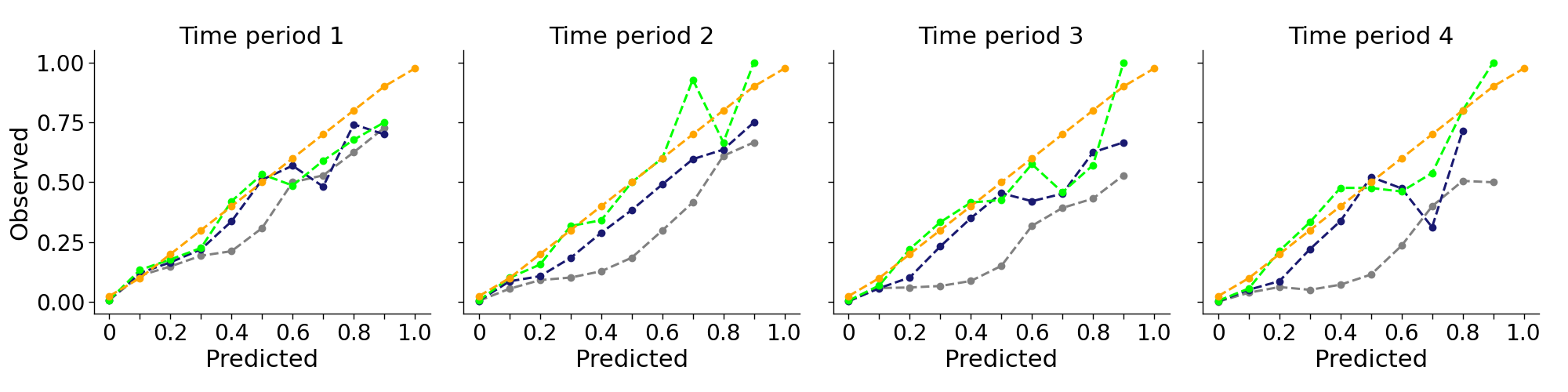}
		\caption{Online logistic revision with respect to a fixed prediction model and patient variables}
		\label{fig:online_linear_copd}
	\end{subfigure}	
	\begin{subfigure}{\linewidth}
		\centering
%	\includegraphics[width=0.29\linewidth]{../code/python/analysis_copd/_output/order_date_cts/nump_36/size_108000/seed_0/combo_boxed/batch_10/regret_0.05/inflat_0,0.005/linear_update_0/thetas_time.png}
%	\hspace{-0.2cm}
\hspace{-1.8cm}
	\includegraphics[width=0.78\linewidth]{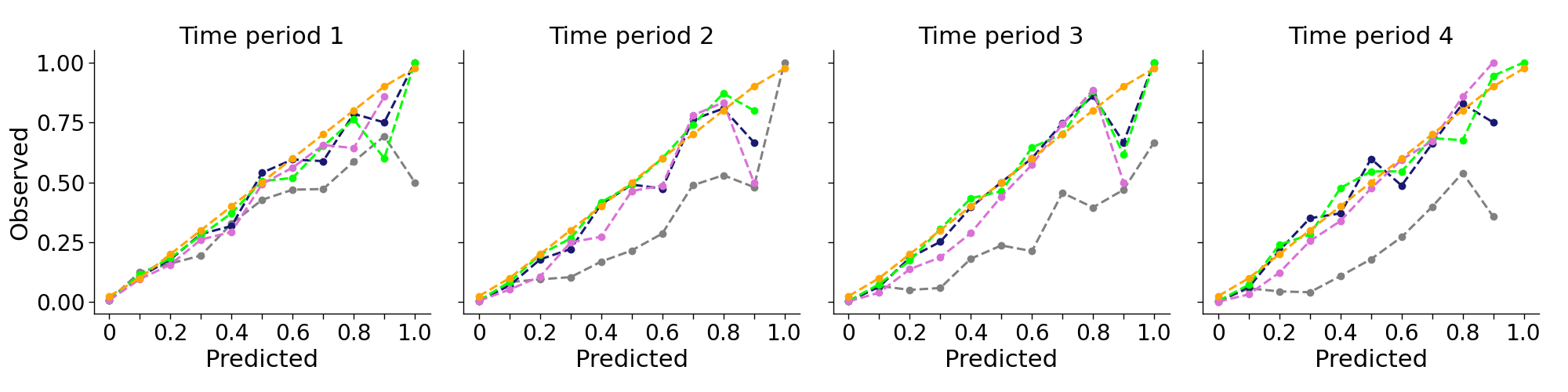}
	
	\includegraphics[width=\linewidth]{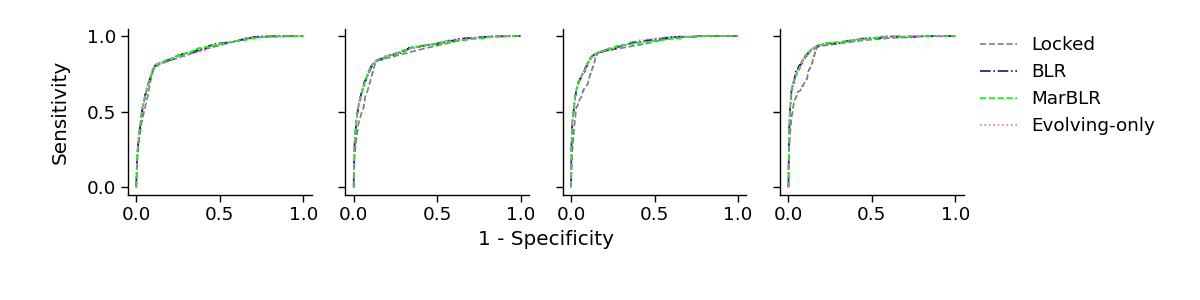}
	\vspace{-0.5cm}
\caption{Online ensembling of the original and continually-refitted prediction models}
\label{fig:online_refit_copd}
	\end{subfigure}
%\begin{subfigure}{\linewidth}
%		\centering
%
%\caption{ROC curves when online wrapper}
%\label{fig:online_refit_copd_roc}
%\end{subfigure}
\caption{Results from online logistic recalibration and revision for a fixed COPD risk prediction model (a and b, respectively) and online reweighting for fixed and continually-refitted (evolving) COPD risk prediction models (c) using BLR and MarBLR. Calibration curves are estimated across four time periods that divide the full dataset into equal lengths. The receiver operating characteristic (ROC) curve is shown for the online ensembling approach because the ROC curves did not change significantly for (a) or (b).
}
\label{fig:copd}
\end{figure}

\pagebreak

\appendix

\section*{Appendix for ``Bayesian logistic regression for online recalibration and revision of clinical prediction models with guarantees''}

\renewcommand\thefigure{A.\arabic{figure}} 
\renewcommand\thetable{A.\arabic{table}} 

\begin{table}
	\centering
	\begin{tabular}{p{2.5in}|p{3.2in}}
		Notation & Description\\
		\toprule
		\multicolumn{2}{c}{\textit{General terms}}\\
		$T$ & Time horizon \\
		$d$ & Number of variables \\
		$(x_t, y_t)$ & Observed variables and outcome at time $t$ \\
		$\hat{f}_t : \mathcal{X} \mapsto \mathbb{R}$ & Underlying prediction model at time $t$\\
		$\hat{A}_t:\mathbb{R} \times \mathcal{X} \mapsto [0,1]$ & Model revision deployed at time $t$: A function that maps predictions from the underlying prediction model at time $t$ and patient variables to a probability  \\
		$\hat{\theta}_t$ & Parameters for logistic model revision at time $t$ \\
		$\tau = (\tau_1, \tau_2, \cdots, \tau_s)$ & Update times for a given sequence of model revisions \\
		\midrule
		\multicolumn{2}{c}{\textit{Regret}}\\
		$
		\bigg(
		\frac{1}{T}\sum_{t=1}^T 
		\bigg [- \log p \left (y_t, \hat{A}_t(\hat{f}_t(x_t), x_t) \right) 
		+ \log p \left (y_t, \hat{A}_1(\hat{f}_1(x_t), x_t) \right)  \bigg ] \bigg )_+
		$ &
		\underline{Type I Regret}: The average increase in the negative log likelihood when using the online reviser instead of locking the original model  \\
		\rule{0pt}{6ex} 
		$
		\bigg(
		\frac{1}{T}\sum_{t=1}^T 
		\bigg [ - \log p\left (y_t, \hat{A}_t(\hat{f}_t(x_t), x_t) \right) 
		+ \log p\left (y_t, {A}^*_{\tau, t}(\hat{f}_t(x_t), x_t) \right)  \bigg ] \bigg)_+
		$ &
		\underline{Type II $\tau$-Regret}: The average increase in the negative log likelihood when using the online reviser versus the oracle model reviser $\{{A}^*_{\tau, t}: t = 1,...,T \}$ with update times $\tau$ 
		\\
		\midrule
		\multicolumn{2}{c}{\textit{BLR and MarBLR parameters}}\\
		$N(\thetainit, \Sigmainit)$ & Gaussian prior in BLR and MarBLR for the logistic revision parameter at time $t = 1$ \\
		$\alpha$ & Prior probability in MarBLR that the model revision shifts at time $t$\\
		$\delta^2$ & Factor controlling the variance of the MarBLR prior over shifts in the model revision parameters
	\end{tabular}
	\caption{
		Terminology and notation
	}
	\label{eq:terminology}
\end{table}

\section{Practical Implementation of BLR and MarBLR}
\label{sec_a:practical}

In this manuscript, we implement MarBLR using a Laplace approximation of the logistic posterior and perform Kalman filtering with collapsing \citep{Gordon1990-jj, West1997-fa}.
Because BLR corresponds to MarBLR with $\alpha = 0$ and $\delta^2 = 0$, we use this same procedure to perform approximate Bayesian inference.
The Kalman filtering approach is simple and computationally efficient; We describe the steps below.
We note that for the special case of BLR, one can also perform posterior inference by sampling Polya-Gamma latent variables \citep{Polson2013-cr}.
This would allow one to perform full Bayesian inference but is significantly more costly in terms of computation time.

We make predictions and update the posterior using the following recursive procedure.
The process is initialized with the Gaussian prior for $\theta_1$ with mean $\thetainit$ and posterior covariance $\Sigmainit$.
Let $D^{(t)}$ denote the observations up to and including time $t$.

\textbf{Prediction step.} At time $t$, let the approximation for $\theta_{t-1}|W_{t-1}= w_{t-1},D^{(t-1)}$ be the Gaussian distribution with mean $\hat{\theta}_{t-1, w_{t-1}}$ and covariance  $\hat{\Sigma}_{t-1, w_{t-1}}$.
We also assume $\Pr(W_{t - 1} = w_{t - 1}|D^{(t-1)})$ is known.
We generate predictions at time $t$ using the posterior distribution $\theta_{t}|D^{(t - 1)}$, which is a mixture of the distributions
\begin{align}
\theta_{t}|W_t= w_t,W_{t - 1} = w_{t-1}, D^{(t - 1)}
\sim
N\left(
\hat{\theta}_{t-1,w_{t-1}},
(1 + \delta^2 w_t) \hat{\Sigma}_{t-1, w_{t-1}}
\right)
\label{eq:predict_parts}
\end{align}
%Our theorems assume $R_t = \delta^2 I$; In practice, we use $R_t = \delta^2 \hat{\Sigma}_{t-1}/\lambda_{\max}(\hat{\Sigma}_{t-1})$ to reflect different levels of uncertainty across parameters.
for $w_t, w_{t=1} \in \{0,1\}$ with weights by $\Pr(W_t = w_t | W_{t-1} = w_{t-1}) \Pr(W_{t - 1} = w_{t - 1}|D^{(t-1)})$. Recall that $\Pr(W_t = 1 | W_{t-1} = w_{t-1}) = \alpha$ in the MarBLR prior.
We predict that $Y = 1$ for a subject $x$ using the posterior mean of $\Pr(Y = 1|X = x)$.

\textbf{Update step.} Next, we observe a new batch of labeled observations and update the posterior.
That is, we must perform inference for $\theta_{t}|D^{(t)}$, which is a mixture of the distributions  $\theta_{t}|W_t= w_t,W_{t-1}= w_{t-1},D^{(t)}$ with probability weights $\Pr(W_t= w_t, W_{t-1} = w_{t-1}|D^{(t)})$ for $w_t, w_{t-1} \in\{ 0,1\}$.
%Suppose the conditional distributions $\theta_{t-1}|W_{t-1}= w_{t-1},D^{(t - 1)}$ follow a Gaussian distribution with mean $\hat{\theta}_{t-1,w_{t-1}}$ and covariance $\hat{\Sigma}_{t-1, w_{t-1}}$.
Let $\tilde{\ell}_t(\theta, w) = \sum_{i=1}^n \log p\left(y_{t,i}|z_{t,i}, \theta \right)  + \log p(\theta\mid w, D^{(t - 1)})$.
We approximate the distribution $\theta_{t}|W_t= w_t,W_{t - 1} = w_{t-1}, D^{(t)}$ using a Gaussian distribution with its mean computed using a Newton update
\begin{align}
\hat{\theta}_{t, w_t, w_{t-1}} = \hat{\theta}_{t-1, w_{t-1}} -
\left[\nabla_{\theta}^2 \tilde{\ell}_t \left(\hat{\theta}_{t-1, w_{t-1}}, w_{t-1}\right)\right]^{-1}
\nabla_{\theta} \tilde{\ell}_t \left(\hat{\theta}_{t-1, w_{t-1}}, w_{t-1} \right)
\end{align}
and its covariance as
$$
\hat{\Sigma}_{t ,w_t, w_{t - 1}} = \left[\nabla_{\theta}^2 \tilde{\ell}_t(\hat{\theta}_{t-1, w_{t-1}}, w_{t-1})\right]^{-1}.
$$
The probability $\Pr(W_t = w_t, W_{t-1} = w_{t-1}\mid D^{(t)})$, which is proportional to
\begin{align}
\left[
\int_{\theta_t}
p\left(y_{t,\cdot} \mid z_{t,\cdot}, \theta_t \right)
p\left(\theta_t \mid w_t, w_{t-1}, D^{(t-1)} \right)
d\theta_t
\right]
\Pr(W_t|W_{t-1})
\Pr(W_{t-1} | D^{(t-1)}),
\label{eq:prob_post}
\end{align}
is approximated using a Laplace approximation for the integral in \eqref{eq:prob_post}, i.e. 
\begin{align}
2\pi^{d/2}
\left| \left\{ \nabla^2_\theta \tilde{\ell}(\hat{\theta}_{t, w_{t}, w_{t - 1}}) \right\}^{-1} \right|^{1/2}
p\left(y_{t, \cdot} \mid z_{t,\cdot}, \hat{\theta}_{t, w_{t}, w_{t - 1}} \right)
p\left(\hat{\theta}_{t, w_{t}, w_{t - 1}} \mid D^{(t-1)} \right).
\label{eq:laplace}
\end{align}
Let $\hat{q}_{w_t, w_{t-1}}$ denote the estimated probability.
Finally, we approximate the posterior distribution $\theta_{t}|W_t= w_t,D^{(t)}$ using a single Gaussian distribution by moment-matching \citep{West1997-fa, Orguner2007-st} with mean and covariance
\begin{align}
\hat{\theta}_{t, w_t} = 
\frac{1}{\hat{q}_{w_t, 0} + \hat{q}_{w_t, 1}} \left(
\hat{q}_{w_t, 0} \hat{\theta}_{t, w_t, 0} + \hat{q}_{w_t, 1} \hat{\theta}_{t, w_t, 1}
\right)
\\
\hat{\Sigma}_{t, w_t} = 
\frac{1}{\hat{q}_{w_t, 0} + \hat{q}_{w_t, 1}} \left(
\hat{q}_{w_t, 0}\hat{\Sigma}_{t, w_t, 0} + \hat{q}_{w_t, 1} \hat{\Sigma}_{t, w_t, 1}
\right).
\end{align}

\section{Online model revision for batched data}
In certain settings, it is more convenient and practical for the data stream to be observed in batches of size $n > 1$.
Here we discuss the necessary modifications to our framework for analyzing the performance on an online model reviser for batched data.
We denote a batch of observations as $\{(x_{t,i}, y_{t,i}): i = 1,....,n\}$ and use the notation $a_{t,\cdot}$  to denote the sequence $(a_{t,1},...,a_{t,n})$.

We extend the online model reviser to output a probability distribution over all possible outcomes for a batch of observations, i.e. $\hat{A}_{t}: \mathbb{R}^n \times \mathcal{X}^n \mapsto \Delta_{2^n}$ where $\Delta_{2^n}$ is the probability simplex over all possible outcomes $(y_{t,1},...,y_{t,n})$.
The loss of the online model reviser over the entire time period is then defined as the average negative log likelihood
\begin{align}
-\frac{1}{nT}\sum_{t = 1}^T \log p\left (y_{t,\cdot},  \hat{A}_{t}\left(\hat{f}_t(x_{t, \cdot}), x_{t,\cdot} \right ) \right ).
\label{eq:batched_loss}
\end{align}
This theoretical framework allows predictions from the online model reviser to depend on all unlabeled observations $\{x_{t,i}: i = 1,...,n\}$.
By defining regret with respect to \eqref{eq:batched_loss}, we are able to derive Type I and II regret bounds for the batched setting.
This is necessary for analyzing BLR and MarBLR because outcomes are not independent given the observations up to time $t$ in the Bayesian framework.
In particular, the outcomes are correlated because of the shared (latent) revision parameter $\theta_t$.

\begin{table}[!htbp] \centering 
	\caption{Descriptive statistics of variables included in the COPD risk prediction model. Continuous variables are summarized by Mean (SD). Binary/ordinal variables are summarized by number of nonzero entries (\%).}
	\label{tab:copd} 
	\begin{tabular}{@{\extracolsep{5pt}}lr} 
		\\[-1.8ex]\hline 
		\hline \\[-1.8ex] 
		Variable & \\ 
		\hline \\[-1.8ex] 
		Diagnosed with COPD & 2756 (2.55) \\ 
		Age at encounter & 60.31 (18.60) \\ 
		\multicolumn{2}{c}{\textit{Medical history}}\\
		Asthma & 741 (0.69) \\ 
		Bronchitis & 5855 (5.42) \\ 
		COPD & 14950 (13.84) \\ 
		Smoking & 42651 (39.49) \\ 
		Pulmonary Function Test & 2844 (2.63) \\ 
		Intubation & 2420 (2.24) \\ 
		Spirometry & 1091 (1.01) \\ 
		Bilevel positive airway pressure & 710 (0.66) \\ 
		Acute coronary syndrome & 11008 (10.19) \\ 
		Pneumonia &15386 (14.25) \\ 
		Steroids & 23249 (21.53) \\ 
		Antihypertensives & 7740 (7.17) \\ 
		Short-acting bronchodilator & 13088 (12.12) \\ 
		Antihistiminic & 17768 (16.45) \\ 
		Respiratory Clearance & 2791 (2.58) \\ 
		Upper Respiratory Infection & 1242 (1.15) \\ 
		Antiarrythmic order & 7650 (7.08) \\ 
		Inhaled bronchodilators & 122 (0.11) \\ 
		Inhaled corticosteroid & 78 (0.07) \\ 
		Long-acting bronchodilator & 91 (0.08) \\ 
		Combination of inhaled bronchodilators & 8 (0.001) \\ 
		\multicolumn{2}{c}{\textit{History of current emergency department visit}} \\
		Pneumonia & 3503 (3.24) \\ 
		Short-acting bronchodilator & 5982 (5.54) \\ 
		Steroids & 4518 (4.18) \\ 
		Antihypertensives & 694 (0.64) \\ 
		Acute coronary syndrome & 2386 (2.21) \\ 
		Antiarrthymic & 1665 (1.54) \\ 
		Antihistaminic & 2624 (2.43) \\ 
		Inhaled corticosteroid & 146 (0.14) \\ 
		Inhaled bronchodilators & 304 (0.28) \\ 
		Long-acting bronchodilator & 420 (0.39) \\ 
		Asthma & 142 (0.13) \\ 
		Upper Respiratory Infection & 238 (0.22) \\ 
		Respiratory Clearance & 131 (0.12) \\ 
		Combination of inhaled bronchodilators & 3 (0.003) \\ 
		\hline \\[-1.8ex] 
	\end{tabular} 
\end{table}

\section{Type I and II Regret bounds}

\subsection{Notation and assumptions}\label{sec:notation}

We suppose there are $n$ observations at time points $t = 1,...,T$ for some $T \ge 2$.
Consider any sequence of revision parameters $\boldsymbol{\theta} = \{\theta_{1}, \theta_{2}, \ldots, \theta_{T}\}$, where $\theta_{t}\in\mathbb{R}^{d}$ for all $t = 1, 2, \ldots, T$,
with unique values at times $\boldsymbol{\tau} = \{\tau_{1},\tau_{2},...,\tau_{s}\}$, where $\tau_{1}=1 < \tau_2 < \ldots < \tau_{s} \leq T$.
In other words, $\{\theta_{\tau_{1}},\theta_{\tau_{2}},\ldots,\theta_{\tau_{s}}\}$ denotes the sequence of values that the sequence $\boldsymbol{\theta}$ shifted over.
%Let $\tau_{s+1}=T+1$ for simplicity. %% [AG: I don't think this is ever used...]
Henceforth, we use $|\boldsymbol{\tau}|$ (rather than $s$) to indicate the number of times in $\boldsymbol{\tau}$.
For ease of notation, we use the convention $\tau_{|\tau|+1} := T + 1$. Note that the variable $\tau_{|\tau|+1}$ is not part of the sequence $\boldsymbol{\tau}$ and is used purely to simplify the notation.
We use $\taulock := \{\tau_1\} = \{1\}$ to denote the shift times in the edge case of ``locked'' sequences $\boldsymbol{\theta}$ that do not shift over time.
Let $D^{(t)}$ denote all the data observed up to time $t$.

The cumulative negative log-likelihood when using Bayesian inference at each time point is
\[
\Lbf = -\sum_{t=1}^{T}\log p\left(y_{t,\cdot}\mid z_{t,\cdot},D^{(t-1)}\right),
\]
where $p\left(y_{t,\cdot}\mid z_{t,\cdot},D^{(t-1)}\right)$ is the posterior distribution at time $t$.
The cumulative negative log-likelihoods for MarBLR and BLR are denoted by $\LMarBLR$ and $\Lblr$, respectively, and are special cases of $\Lbf$ for their specific choice of priors.
The MarBLR prior $p_0$ over $\boldsymbol{\theta}$ is defined using a Gaussian random walk with a homogeneous transition matrix as follows.
Given $\thetainit \in \R^d$, $\Sigmainit \in \R^{d\times d}$, and some shift probability $\alpha\in[0,1]$, let
\begin{equation}
\theta_{1} \sim N(\thetainit,\Sigmainit), \quad W_1 = 1,
\label{eq:MarBLR_prior_theta_1}
\end{equation}
and for $t = 2, 3, \ldots, T$ let
\begin{equation}
\begin{aligned}
\theta_{t} &=\theta_{t-1}+\beta_{t}W_{t}\\
W_{t} &\sim \Bernoulli(\alpha)\\
\beta_{t} &\sim N(0,\delta^{2}\Sigmainit).
\end{aligned}
\label{eq:MarBLR_prior_theta_t}
\end{equation}
Note that $\boldsymbol{\tau} = \left\{\tau_1, \tau_2, \ldots, \tau_{|\boldsymbol{\tau}|}\right\}$ can be regarded as the indices at which the sequence $\left\{W_1, W_2, \dots, W_T\right\}$ is $1$-valued. In particular, having $\boldsymbol{\tau} = \taulock$ implies that $W_1 = 1$ and $W_t = 0$ for all $t>1$.
The BLR prior is a special case where $\delta^2 = \alpha = 0$.

Type I regret compares BLR and MarBLR to locking the original revision parameters at its initial value $\thetainit$, i.e. $\theta_t = \thetainit$ for all $t\in\left\{1, 2, \ldots, T\right\}$.
The cumulative negative log-likelihood of the locked initial model is given by
\begin{equation*}
\Llocked = -\sum_{t=1}^{T}\log p\left(y_{t,\cdot}\mid z_{t,\cdot};\thetainit\right).
\end{equation*}

Type II $\boldsymbol{\tau}$-regret compares BLR and MarBLR to the best sequence of parameters in retrospect for update times $\boldsymbol{\tau}$, denoted $\tilde{\theta}_{\tau_{j}}$ for $j=1,...,|\tau|$.
Its cumulative negative log-likelihood is defined as
\[
\Ldyntau = -\sum_{t=1}^{T} \log p\left(y_{t,\cdot}\mid z_{t,\cdot};\tilde{\theta}_{t}\right)
\]
where $\tilde{\theta}_{\tau_{j}}$ for $j=1,...,|\tau|$ satisfy 
\begin{equation}
\left.\nabla\sum_{t=\tau_{j}}^{\tau_{j+1}-1}\log p\left(y_{t,\cdot}\mid z_{t,\cdot};\theta\right)\right|_{\theta=\tilde{\theta}_{\tau_{j}}}=0,\quad\forall j=1,\ldots, |\boldsymbol{\tau}|.
\label{eq:theta_tilde}
\end{equation}

In addition, we introduce the notion of a distribution over the sequences
$\boldsymbol{\theta}$. For such a distribution
$Q$, its expected negative log-likelihood is given by
\[
\Lq = E_{Q}\left[-\sum_{t=1}^{T}\log p\left(y_{t,\cdot}\mid z_{t,\cdot};\theta_{t}\right)\right].
\]
Given mean and variance parameters $\boldsymbol{\mu} = (\mu_{t})_{t\in\boldsymbol{\tau}}$ and $\boldsymbol{\Sigma} = (\Sigma_{t})_{t\in\boldsymbol{\tau}}$, we define $\Qtaumusigma$ to be the distribution over $\boldsymbol{\theta}$ with shift times $\boldsymbol{\tau}$ where $\theta_{\tau_{j}}$ for $j=1,...,|\boldsymbol{\tau}|$ are jointly independent and normally distributed per
\begin{equation}\label{eq:Qtaumusigma}
\theta_{\tau_{j}}\sim N\left(\mu_{\tau_{j}},\Sigma_{j}\right).
\end{equation}

Some results in the following sections rely on the assumption that there exists a constant $c>0$ such that
\begin{equation}\label{eq:second_derivative_assumption}
\left|\frac{\partial^{2}}{\partial w^{2}}\log p\left(y | z^\top \theta = w\right)\right|\le c,
\end{equation}
for all $y$ and $w$.
This always holds for logistic regression with $c \leq 1$.

\subsection{Useful Results}

Consider the prior distribution $p_{0}(\boldsymbol{\theta})$ over sequences $\boldsymbol{\theta}$.
Let $p_{0}(\boldsymbol{\tau})$ be its marginal distribution over shift times $\boldsymbol\tau$ and  $p_{0}\left(\boldsymbol{\theta} \mid \boldsymbol{\tau}\right)$ be the conditional distribution over sequences $\boldsymbol{\theta}$ with shift times $\boldsymbol{\tau}$.

\begin{lem}[Variational bound]\label{lemma:variational_bound_general}
	Consider any prior distribution $p_{0}$ over sequences $\boldsymbol{\theta}$.
	Given any $\boldsymbol{\tau}$ and any distribution $Q$, it holds that
	\begin{equation*}
	\Lbf - \Lq \le \E_{\boldsymbol{\tau} \sim Q}\left[
	\KL\left(Q(\boldsymbol{\theta} | \boldsymbol{\tau}) \mid\mid p_{0}\left(\boldsymbol{\theta} \mid \boldsymbol{\tau}\right)\right)
	\right]
	+  \KL\left(Q(\boldsymbol{\tau}) \mid\mid p_{0}\left(\boldsymbol{\tau}\right)\right).
	\end{equation*}
\end{lem}

\begin{proof}
	First, we can reexpress the cumulative negative log-likelihood of the Bayesian dynamical model by chaining the conditional probabilities as follows:
	\begin{align*}
	\Lbf &= -\sum_{t=1}^{T}\log p\left(y_{t,\cdot}\mid z_{t,\cdot},D^{(t-1)}\right) \\
	&= -\log p\left((y_{t,\cdot})_{t=1,...,T}\mid(z_{t,\cdot})_{t=1,...,T}\right).
	\end{align*}
	Similarly, the cumulative negative log-likelihood of any sequence of calibration parameters can be written as
	\[
	-\sum_{t=1}^{T} \log p\left(y_{t,\cdot}\mid z_{t,\cdot};\theta_{t}\right)=-\log p\left((y_{t,\cdot})_{t=1,...,T}\mid(z_{t,\cdot})_{t=1,...,T}; \boldsymbol{\theta}\right).
	\]
	Thus, the difference in the cumulative negative log-likelihood between the Bayesian dynamical model and any sequence of parameters is given by
	\begin{align*}
	\Lbf - \Lq &= \E_{Q} \left[\log\frac{p\left((y_{t,\cdot})_{t=1,...,T}\mid(z_{t,\cdot})_{t=1,...,T}; \boldsymbol{\theta}\right)}{p\left((y_{t,\cdot})_{t=1,...,T}\mid(z_{t,\cdot})_{t=1,...,T}\right)}\right].
	\end{align*}
	By Bayes' Rule, the posterior distribution $p_{T}$ over $\boldsymbol{\theta}$ with respect to the Bayesian dynamical model satisfies
	\[
	p_{T}\left(\boldsymbol{\theta}\right) = \frac{p\left((y_{t,\cdot})_{t=1,...,T}\mid(z_{t,\cdot})_{t=1,...,T}; \boldsymbol{\theta} \right) p_{0}\left(\boldsymbol{\theta}\right)}{p\left((y_{t,\cdot})_{t=1,...,T}\mid(z_{t,\cdot})_{t=1,...,T}\right)}.
	\]
	Thus, we have that
	\begin{equation}\label{eq:Lbf-Lq-2}
	\begin{aligned}
	\Lbf - \Lq &= \E_{Q}\left[\log\frac{p_{T}\left(\boldsymbol{\theta}\right)}{p_{0}\left(\boldsymbol{\theta}\right)}\right]\\
	&= \E_{\boldsymbol{\tau}\sim Q}\left[
	\E_{\boldsymbol{\theta}\sim Q(\cdot | \boldsymbol{\tau})}\left[
	\log\frac{p_{T}\left(\boldsymbol{\theta}\mid\boldsymbol{\tau}\right)}{p_{0}\left(\boldsymbol{\theta}\mid\boldsymbol{\tau}\right)}
	\right]\right]
	+\E_{\boldsymbol{\tau}\sim Q}\left[\log \frac{p_{T}(\boldsymbol{\tau})}{p_{0}(\boldsymbol{\tau})}\right].
	\end{aligned}
	\end{equation}
	Moreover, because the KL divergence is always positive, it holds that
	\begin{equation}\label{eq:KLQp0-2}
	\begin{aligned}
	\E_{\boldsymbol{\theta}\sim Q(\cdot | \boldsymbol{\tau})}\left[\log\frac{p_{T}\left(\boldsymbol{\theta}\mid\boldsymbol{\tau}\right)}{p_{0}\left(\boldsymbol{\theta}\mid\boldsymbol{\tau}\right)}\right]
	&= \KL\left(Q(\cdot | \boldsymbol{\tau}) \mid\mid p_{0}\left(\boldsymbol{\theta}\mid\boldsymbol{\tau}\right)\right) - \KL\left(Q(\cdot | \boldsymbol{\tau}) \mid\mid p_{T}\left(\boldsymbol{\theta}\mid\boldsymbol{\tau}\right)\right)\\
	&\le \KL\left(Q(\cdot | \boldsymbol{\tau}) \mid\mid p_{0}\left(\boldsymbol{\theta}\mid\boldsymbol{\tau}\right)\right).
	\end{aligned}
	\end{equation}
	Likewise,
	\begin{align}
	E_{\boldsymbol{\tau}\sim Q}\left[\log \frac{p_{T}(\boldsymbol{\tau})}{p_{0}(\boldsymbol{\tau})}\right]
	\le \KL\left(
	Q(\boldsymbol{\tau}) \mid\mid p_{0}(\boldsymbol{\tau})
	\right).
	\end{align}
	
	Finally, by combining equations (\ref{eq:Lbf-Lq-2}) and (\ref{eq:KLQp0-2}) we arrive at the conclusion of this theorem.
\end{proof}

\subsection{Type I regret results for MarBLR}

Let the distribution $p_0$ be the MarBLR prior as defined per \eqref{eq:MarBLR_prior_theta_1} and \eqref{eq:MarBLR_prior_theta_t}.
For a given $\boldsymbol{\tau}$, let $\Qwalk$ be a Gaussian random walk with expected shifts $\beta_j$ at known shift times $\tau_j$ for $j = 1,...,|\tau|$.
That is,
\begin{align}
\theta_{\tau_{j}} - \theta_{\tau_{j - 1}} \sim N\left(\beta_j, \nu^2 \Sigmainit \right).
\end{align}
and 
\begin{align}
\theta_{\tau_{1}} \sim N\left(\mu_1, \epsilon_1^2 \Sigmainit \right).
\end{align}

We begin with simplifying the KL divergence term in Lemma~\ref{lemma:variational_bound_general}.

\begin{lem}\label{lemma:KL_gauss_rand_walk}
	For any $\boldsymbol{\tau}$, consider the Gaussian random walk $\Qwalk$.
	We have that
	\begin{align}
	\label{eq:lemma_KL_gauss_rand_walk}
	\begin{split}
	\KL\left(\Qwalk \mid\mid p_{0}\left(\boldsymbol{\theta} \mid \boldsymbol{\tau}\right)\right)
	=&  
	\frac{1}{2}\epsilon_1^{2}d+\frac{1}{2}\left(\mu_1-\thetainit\right)^{\top}\Sigmainit^{-1}\left(\mu_1-\thetainit\right)-\frac{d|\tau|}{2}\\
	& + d(|\tau| - 1) \log \frac{\delta}{\nu} - d\log \epsilon_1 \\
	& + \frac{d \nu^2 (|\tau| - 1)}{2 \delta^2} + \frac{1}{2\delta^{2}}
	\sum_{j=2}^{|\tau|} \beta_{j}^\top \Sigmainit^{-1} \beta_{j}.
	\end{split}
	\end{align}
\end{lem}

\begin{proof}
	
	For ease of notation, let $\Theta_{J}$ be the space over sequences
	$\left(\theta_{1},...,\theta_{J}\right)$.	
	Given the known times $\boldsymbol{\tau}$, there is a one-to-one mapping from sequences
	in $\Theta_{|\boldsymbol{\tau}|}$ to sequences in $\Theta_{T}$ with unique values
	at times $\boldsymbol{\tau}$.
	Let $\Qsub$ be the probability distribution over $\Theta_{|\boldsymbol{\tau}|}$
	as defined by $\Qwalk$. Likewise, let $\psub$
	be the PDF over $\Theta_{|\boldsymbol{\tau}|}$ as defined by the conditional
	prior distribution $p_{0}\left(\cdot\mid\boldsymbol{\tau}\right)$.
	
	We have that 
	\begin{align}
	&\KL\left(\Qwalk \mid\mid p_{0}\left(\boldsymbol{\theta} \mid \boldsymbol{\tau}\right)\right) \nonumber \\
	=&\int \cdots \int \Qsub(\boldsymbol{\theta}) 
	\log \frac{\Qsub(\boldsymbol{\theta})}{\psub(\boldsymbol{\theta})}
	\mathrm{d}\boldsymbol{\theta} \nonumber \\
	\begin{split}
	=& \int \Qsub({\theta}_1) 
	\log \frac{\Qsub({\theta}_1)}{\psub({\theta}_1)}
	\mathrm{d}{\theta}_1  + \sum_{j=1}^{|\tau| -1}
	\int \int \Qsub({\theta}_{j}, {\theta}_{j + 1}) 
	\log \frac{\Qsub({\theta}_{j + 1} \mid {\theta}_{j})}{\psub({\theta}_{j + 1}\mid {\theta}_{j})}
	\mathrm{d}{\theta}_{j + 1} \mathrm{d}{\theta}_{j}
	\end{split}
	\label{eq:kl_general_2}
	\end{align}
	The first term in \eqref{eq:kl_general_2} is the KL divergence of two multivariate Normal distributions, $N(\mu_1, \epsilon_1^2 \Sigmainit)$ and $N(\thetainit, \Sigmainit)$, and can be shown to be equal to
	\begin{equation}
	\frac{1}{2}\epsilon_1^{2}d+\frac{1}{2}\left(\mu_1-\thetainit\right)^{\top}\Sigmainit^{-1}\left(\mu_1-\thetainit\right)-d\log\epsilon_1-\frac{d}{2}.
	\label{eq:KL_intermed_1}
	\end{equation}
	Also, for $j =1,...,|\tau| -1$, we have that each summand in the second term in \eqref{eq:kl_general_2} is equal to
	\begin{align}
	& \int \int \Qsub({\theta}_{j}, {\theta}_{j + 1}) 
	\log \frac{\Qsub({\theta}_{j + 1}\mid {\theta}_{j})}{\psub({\theta}_{j + 1}\mid {\theta}_{j})}
	\mathrm{d}{\theta}_{j + 1} \mathrm{d}{\theta}_{j} \nonumber \\
	%	\begin{split}
	%	=&  
	%	\frac{1}{2}\log \left |2\pi \delta^2 \Sigmainit \right | \\
	%	&\quad+
	%	\frac{1}{2\delta^{2}}
	%	\int \int \Qsub({\theta}_{j}, {\theta}_{j + 1}) 
	%	\left(
	%	{\theta}_{j + 1} -  {\theta}_{j}
	%	\right)^\top
	%	\Sigmainit^{-1}
	%	\left(
	%	{\theta}_{j + 1} -  {\theta}_{j}
	%	\right)
	%	\mathrm{d}{\theta}_{j + 1} \mathrm{d}{\theta}_{j}\\
	%	&\quad- \frac{1}{2} \log \left |2\pi e \nu^2 \Sigmainit\right |
	%	\end{split} \nonumber \\
	= &
	d \log\frac{\delta}{\nu}
	+\frac{1}{2\delta^{2}}
	\int \int \Qsub({\theta}_{j}, {\theta}_{j + 1}) 
	\left(
	{\theta}_{j + 1} -  {\theta}_{j}
	\right)^\top
	\Sigmainit^{-1}
	\left(
	{\theta}_{j + 1} -  {\theta}_{j}
	\right)
	\mathrm{d}{\theta}_{j + 1} \mathrm{d}{\theta}_{j}
	-\frac{d}{2}.
	\label{eq:qq_intermed_2}
	\end{align}
	By the definition of $\Qwalk$, we have that $\Delta_{j + 1} = \theta_{\tau_{j + 1}} - \theta_{\tau_j} \sim N\left (\beta_{j + 1}, \nu^2 \Sigmainit \right )$.
	Thus,
	\begin{align*}
	\int \Qsub({\theta}_{j}, {\theta}_{j + 1}) 
	\left(
	{\theta}_{j + 1} -  {\theta}_{j}
	\right)^\top
	\Sigmainit^{-1}
	\left(
	{\theta}_{j + 1} -  {\theta}_{j}
	\right)
	\mathrm{d}{\theta}_{j + 1} \mathrm{d}{\theta}_{j}
	=& d\nu^2 + \beta_{j+1}^\top \Sigmainit^{-1} \beta_{j+1}.
	\end{align*}
	Plugging in the above result into \eqref{eq:qq_intermed_2}, we have that
	\begin{align}
	\int \int \Qsub({\theta}_{j}, {\theta}_{j + 1}) 
	\log \frac{\Qsub({\theta}_{j + 1} \mid {\theta}_{j})}{\psub({\theta}_{j + 1}\mid {\theta}_{j})}
	\mathrm{d}{\theta}_{j + 1} \mathrm{d}{\theta}_{j} 
	= &
	d \log\frac{\delta}{\nu}
	+\frac{1}{2\delta^{2}}
	\left(
	d\nu^2 + \beta_{j+1}^\top \Sigmainit^{-1} \beta_{j+1}
	\right)
	-\frac{d}{2}.
	\label{eq:KL_intermed_2}
	\end{align}
	Combining the results \eqref{eq:kl_general_2}, \eqref{eq:KL_intermed_1} and \eqref{eq:KL_intermed_2}, we attain the desired conclusion.
\end{proof}

To bound the Type I regret for MarBLR, we compare the regret via the intermediary $Q$ with marginal distribution over $\boldsymbol{\tau}$ the same as $p_0$ and the conditional distribution given $\tau$ to be $\Qwalk$ with $\boldsymbol{\beta}_j = 0$ for all $j = 2,...,|\tau|$.
That is, the regret is decomposed into
\begin{align}
(\Lbf - \Lq) + (\Lq - \Llocked).
\end{align}

We bound $\Lbf - \Lq$ by marginalizing Lemma~\ref{lemma:KL_gauss_rand_walk} over $\boldsymbol{\tau}$ as follows.
\begin{lem}\label{lemma:KL_margin_walk}
	Let the distribution $p_0$ be defined as above.
	Let distribution $Q$ over $\boldsymbol{\theta}$ have the same distribution over $\boldsymbol{\tau}$ as $p_0$, with $\theta_{1}$ distributed $N(\thetainit, \epsilon_{1}^2 \Sigmainit)$, and $Q(\cdot | \boldsymbol{\tau})$ be a zero-centered Gaussian random walk $\boldsymbol{\beta}_j = 0$ for all $j = 2,...,|\tau|$.
	Let $\xi = \E_{p_0}|\boldsymbol{\tau}|$.
	We have that
	\begin{align}
	\Lbf - \Lq
	\le & 
	\frac{1}{2}\epsilon_1^{2}d -\frac{d \xi}{2} + d(\xi - 1) \log\frac{\delta}{\nu} - d\log \epsilon_1  + \frac{d \nu^2 (\xi - 1)}{2 \delta^2}.
	\end{align}
\end{lem}
\begin{proof}
	Taking the expectation of \eqref{eq:lemma_KL_gauss_rand_walk} from Lemma~\ref{lemma:KL_gauss_rand_walk} with respect to $\boldsymbol{\tau}$ under the additional assumptions of this Lemma, and plugging the result into Lemma~\ref{lemma:variational_bound_general} yields the desired conclusion.
\end{proof}

Next we bound $\Lq - \Llocked$.
\begin{lem}\label{lemma:Lq_minus_Llocked_general}
	Assume that there is a $c>0$ that bounds the second derivative as in (\ref{eq:second_derivative_assumption}).
	Assume that there is an $R$ such that $\frac{1}{n(\tau_{j+1}-\tau_{j})}\sum_{t=\tau_{j}}^{\tau_{j+1}-1}\sum_{i=1}^{n}z_{t,i}z_{t,i}^{\top}\preceq R^{2}I$ for all $j\in\{1, 2, \ldots, |\boldsymbol{\tau}|\}$.
	Let $\Qwalk$ be the zero-centered Gaussian random walk with $\mu_1 = \thetainit$.
	Then it holds that
	\[
	\Lqwalk - \Llocked \le
	\frac{cnR^{2} }{2}  \Tr(\Sigmainit) \left(
	T \epsilon_1^{2} + 
	\nu^2 \sum_{j=2}^{|\tau|}(\tau_{j + 1} - \tau_{j})(j - 1) 
	\right ).
	\]
\end{lem}

\begin{proof}
	We use a Taylor expansion. For $j = 1,...,|\tau|$, there is some $\theta_{mid}$ such that
	\begin{align}
	\begin{split}
	-\sum_{t=\tau_j}^{\tau_{j + 1} - 1}\sum_{i=1}^{n} \log p\left(y_{t,i}\mid z_{t,i};\theta_{\tau_j} \right)
	=&-\sum_{t=\tau_j}^{\tau_{j + 1} - 1}\sum_{i=1}^{n}\log p\left(y_{t,i}\mid z_{t,i};\thetainit\right)  -\left.\nabla_{\theta}\sum_{t=\tau_j}^{\tau_{j + 1} - 1}\sum_{i=1}^{n}\log p\left(y_{t,i}\mid z_{t,i};\theta\right)\right|_{\theta=\thetainit}^{\top}\left(\theta_{\tau_j}-\thetainit\right)\\
	& \ -\frac{1}{2}\left(\theta_{\tau_j}-\thetainit\right)^{\top}\left.\nabla_{\theta}^{2}\sum_{t=\tau_j}^{\tau_{j + 1} - 1}\sum_{i=1}^{n}\log p\left(y_{t,i}\mid z_{t,i};\theta\right)\right|_{\theta=\theta_{mid}}\left(\theta_{\tau_j}-\thetainit\right).
	\end{split}
	\label{eq:taylor_expansion_general}
	\end{align}
	Note that
	\[
	\left(\theta_{\tau_j}-\thetainit\right)^{\top}\nabla_{\theta}^{2}\log p\left(y\mid z;\theta\right)\left(\theta_{\tau_j}-\thetainit\right)=\frac{\partial^{2}}{\partial w^{2}}\log p\left(y|w\right)\left(z^{\top}\left(\theta_{\tau_j}-\thetainit\right)\right)^{2},
	\]
	where $w=z^{\top}\theta$ is the predicted logit. Using equation (\ref{eq:second_derivative_assumption}) it follows that
	\begin{align}
	\label{eq:quadratic_form_general}
	& \left|\frac{1}{2}\left(\theta_{\tau_j}-\thetainit\right)^{\top}\left.\nabla_{\theta}^{2}\sum_{t=\tau_j}^{\tau_{j + 1} - 1}\sum_{i=1}^{n}\log p\left(y_{t,i}\mid z_{t,i};\theta\right)\right|_{\theta=\theta_{mid}}\left(\theta_{\tau_j}-\thetainit\right)\right|\\
	\le&\frac{c}{2}\sum_{t=\tau_j}^{\tau_{j + 1} - 1}\sum_{i=1}^{n}\left(z_{t,i}^{\top}\left(\theta_{\tau_j}-\thetainit\right)\right)^{2}\\
	= &\frac{c}{2}\left(\theta_{\tau_j}-\thetainit\right)^\top \left(\sum_{t=\tau_j}^{\tau_{j + 1} - 1}\sum_{i=1}^{n}z_{t,i}z_{t,i}^{\top}\right)\left(\theta_{\tau_j}-\thetainit\right).
	\end{align}
	Because the expected value of $\theta$ with respect to $Q$ is $\thetainit$, we have the following after taking the expectation of equation (\ref{eq:taylor_expansion_general}) combined with equation (\ref{eq:quadratic_form_general}):
	
	\begin{align*}
	\Lq &= \E_{Q}\left[-\sum_{t=1}^{T}\sum_{i=1}^{n}\log p\left(y_{t,i}\mid z_{t,i};\theta_t\right)\right] \\
	&\leq \Llocked + \sum_{j=1}^{|\tau|} \frac{c}{2} \E_{Q}\left[\left(\theta_{\tau_j}-\thetainit\right)^{\top}\left(\sum_{t=\tau_j}^{\tau_{j + 1} - 1}\sum_{i=1}^{n}z_{t,i}z_{t,i}^{\top}\right)\left(\theta_{\tau_j}-\thetainit\right)\right].
	\end{align*}
	
	Assuming there exists some $R^2$ that satisfies the lemma assumptions, the following holds after taking the expectation with respect to $Q$:
	\begin{align*}
	\E_{Q}\left[\left(\theta_{\tau_j}-\thetainit\right)^{\top}\left(\sum_{t=\tau_j}^{\tau_{j + 1} - 1}\sum_{i=1}^{n}z_{t,i}z_{t,i}^{\top}\right)\left(\theta_{\tau_j}-\thetainit\right)\right]
	&\le (\tau_{j + 1} - \tau_{j})nR^{2}  \E_{Q}\|\theta_{\tau_j}-\thetainit\|^{2}\\
	&= (\tau_{j + 1} - \tau_{j}) nR^{2} \left(\epsilon_1^{2} + (j - 1) \nu^2 \right) \Tr(\Sigmainit).
	\end{align*}
	After summing over $j = 1,...,|\tau|$, we reach our desired result.
\end{proof}

We combine the two prior lemmas to obtain the following bound on the Type I error for MarBLR.

\begin{thm}[Type I regret for MarBLR]
	Let $\xi = \E_{p_0} |\boldsymbol{\tau}|$ denote the expected number of shift times be denoted.
	The Type I regret for MarBLR is bounded as follows:
	\begin{align*}
	\LMarBLR - \Llocked
	&\le
	\frac{d}{2} \log \left(
	1 + \frac{cnR^{2} T  \Tr(\Sigmainit)}{d}
	\right) \\
	&\quad\quad +
	\frac{d \alpha (T - 1)}{2} \log \left(
	1 + \frac{\delta^2 c nR^{2} T \Tr(\Sigmainit)}{2 d}
	\right).
	\end{align*}
	\label{thrm:type_i_marblr}
\end{thm}

\begin{proof}
	First, note that under the MarBLR prior $p_0$ over shift times $\boldsymbol{\tau}$ as defined previously, we have that
	$$
	\E_{p_0} \left[
	\sum_{j=2}^{|\tau|}(\tau_{j + 1} - \tau_{j})(j - 1) 
	\right]
	=
	\E_{p_0}\left[\sum_{t=2}^T W_t (T + 1 - t)\right]
	= \frac{\alpha}{2} T(T - 1),
	$$
	% [here is why:
	% - the second sum starts from t=2 because W_1 = 1 is fixed.
	% - note that the factor (T + 1 - t) comes from the fact that \tau_{|\tau|+1} = T+1.
	% - we have E(sum_2^T W_t (T+1-t)) = alpha ((T-1)(T+1) - sum_2^T t) = alpha ((T-1)(T+1) - T(T+1)/2 + 1) = alpha (T(T+1) - T - 1 - T(T+1)/2 + 1) = alpha (T(T+1)/2 - T) = alpha (T(T-1)/2)
	% ]
	and $\xi = \E_{p_{0}}|\boldsymbol{\tau}| = \alpha (T - 1) + 1$.
	
	Thus, summing the upper bounds from Lemmas~\ref{lemma:KL_margin_walk} and \ref{lemma:Lq_minus_Llocked_general} and taking expectations with respect to $\boldsymbol{\tau} \sim p_0$, we have that
	\begin{align}
	\Lbf - \Llocked
	& \le
	\frac{1}{2}\epsilon_1^{2}d -\frac{d \alpha (T - 1)}{2} - \frac{d}{2} + d \alpha (T - 1) \log\frac{\delta}{\nu} - d\log \epsilon_1  + \frac{d  \alpha (T - 1)}{2 \delta^2}\nu^2
	\\
	&
	+ \frac{cnR^{2} T}{2}  \Tr(\Sigmainit) \left(
	\epsilon_1^{2} + 
	\frac{\alpha}{2} (T - 1) \nu^2 
	\right ).
	\end{align}
	We minimize the upper bound by selecting
	$$
	\epsilon_1^2 = \frac{d}{d + cnR^{2} T  \Tr(\Sigmainit)}
	$$
	and
	$$
	\nu^2 = \frac{d}{\frac{d}{\delta^2} + \frac{c}{2} nR^{2} T \Tr(\Sigmainit)}
	$$
	to obtain the upper bound
	$$
	\frac{d}{2} \log \left(
	1 + \frac{cnR^{2} T  \Tr(\Sigmainit)}{d}
	\right)
	+
	\frac{d \alpha (T - 1)}{2} \log \left(
	1 + \frac{\delta^2 c nR^{2} T \Tr(\Sigmainit)}{2 d}
	\right).
	$$
\end{proof}

\subsection{Type II $\boldsymbol{\tau}$-regret results for BLR}

Let $\thetataulock$ be the minimizer of the cumulative log-likelihood of the locked model, i.e., $\thetataulock$ satisfies that
\[
\left.\nabla\sum_{t=1}^{T}\sum_{i=1}^{n}\log p\left(y_{t,i}\mid z_{t,i};\theta\right)\right|_{\theta=\tilde{\theta}_{\tau_{\mathrm{locked}}}}=0.
\]
Let $\tilde{Q}$ denote the distribution $\Qtaulocked$ (defined according to section \ref{sec:notation} and equation (\ref{eq:Qtaumusigma}) with the parameters specified here). That is, we have that
\[
\tilde{Q}(\theta_1) = \Qtaulocked(\theta_1)=N\left(\tilde{\theta}_{\tau_{locked}},\epsilon^{2}\Sigmainit\right),
\]
and $\theta_t = \theta_1$ for all $t\in\left\{2, 3, \ldots, T\right\}$.

We bound the difference in the cumulative negative log-likelihood,
$\Lblr - \Ldyntau$, by breaking it into two summands

\begin{equation}\label{eq:blr_decomposition}
\Lblr - \Ldyntau = \left(\Lblr- L_{\tilde{Q}}\right) + \left(L_{\tilde{Q}} - \Ldyntau\right).
\end{equation}

We have already bounded the first summand by Lemmas~\ref{lemma:variational_bound_general} and \ref{lemma:KL_gauss_rand_walk}. We just need to bound the second summand.

\begin{lem}\label{lemma:blr}
	Assume that the second derivative is bounded by a constant $c$ as shown in equation (\ref{eq:second_derivative_assumption}), and that there are $R_{\tau_1}, R_{\tau_2}, \ldots, R_{\tau_{|\boldsymbol{\tau}|}}$ such that
	$$\frac{1}{n(\tau_{j+1}-\tau_{j})}\sum_{t=\tau_{j}}^{\tau_{j+1}-1}\sum_{i=1}^{n}z_{t,i}z_{t,i}^{\top}\preceq R_{j}^{2}I.$$
	It holds that
	\begin{align*}
	L_{\tilde{Q}} - \Ldyntau
	& \le \frac{cn \sum_{j=1}^{|\boldsymbol{\tau}|}R_j^2\left(\tau_{j+1}-\tau_{j}\right)}{2} \epsilon^{2}\Tr(\Sigmainit)
	+ \frac{cn}{2}\sum_{j=1}^{|\boldsymbol{\tau}|}R_j^2\left(\tau_{j+1}-\tau_{j}\right)\left\Vert \tilde{\theta}_{\tau_{locked}}-\tilde{\theta}_{\tau_j}\right\Vert ^{2}.
	\end{align*}
\end{lem}

\begin{proof}
	Because $\tilde{\theta}_{\tau_j}$ is the minimizer of $\nabla_{\theta}\sum_{t=\tau_{j}}^{\tau_{j+1}-1}\sum_{i=1}^{n}\log p\left(y_{t,i}\mid z_{t,i};\theta\right)$, per Taylor's expansion there is some $\theta_{\mathrm{mid}}$ such that
	\begin{multline*}
	-\sum_{t=\tau_{j}}^{\tau_{j+1}-1}\sum_{i=1}^{n}\log p\left(y_{t,i}\mid z_{t,i};\theta\right) = -\sum_{t=\tau_{j}}^{\tau_{j+1}-1}\sum_{i=1}^{n}\log p\left(y_{t,i}\mid z_{t,i};\tilde{\theta}_{\tau_j}\right)\\
	-\frac{1}{2}\left(\theta-\tilde{\theta}_{\tau_j}\right)^{\top}\left.\nabla_{\theta}^{2}\sum_{t=\tau_{j}}^{\tau_{j+1}-1}\sum_{i=1}^{n}\log p\left(y_{t,i}\mid z_{t,i};\theta\right)\right|_{\theta=\theta_{\mathrm{mid}}}\left(\theta-\tilde{\theta}_{\tau_j}\right).
	\end{multline*}
	
	Following the same arguments as in the proof of Lemma \ref{lemma:Lq_minus_Llocked_general}, we have that 
	\begin{multline*}
	\E_{\tilde{Q}}\left[-\sum_{t=\tau_{j}}^{\tau_{j+1}-1}\sum_{i=1}^{n}\log p\left(y_{t,i}\mid z_{t,i};\theta\right)\right]
	\le-\sum_{t=\tau_{j}}^{\tau_{j+1}-1}\sum_{i=1}^{n}\log p\left(y_{t,i}\mid z_{t,i};\tilde{\theta}_{\tau_j}\right) \\
	+\frac{c}{2} \E_{\tilde{Q}}\left[\left(\theta_1-\tilde{\theta}_{\tau_j}\right)^{\top}\left(\sum_{t=\tau_{j}}^{\tau_{j+1}-1}\sum_{i=1}^{n}z_{t,i}z_{t,i}^{\top}\right)\left(\theta_1-\tilde{\theta}_{\tau_j}\right)\right].
	\end{multline*}
	Taking expectation with respect to $\tilde{Q}$, we note that
	\begin{align*}
	& \E_{\tilde{Q}}\left[\left(\theta_1-\tilde{\theta}_{\tau_j}\right)^{\top} \left(\sum_{t=\tau_{j}}^{\tau_{j+1}-1}\sum_{i=1}^{n}z_{t,i}z_{t,i}^{\top}\right) \left(\theta_1-\tilde{\theta}_{\tau_j}\right)\right]\\
	& = \E_{\tilde{Q}}\left[\left(\theta_1-\thetataulock\right)^{\top} \left(\sum_{t=\tau_{j}}^{\tau_{j+1}-1}\sum_{i=1}^{n}z_{t,i}z_{t,i}^{\top}\right) \left(\theta_1-\thetataulock\right)\right]
	+\left(\thetataulock-\tilde{\theta}_{\tau_j}\right)^{\top}\left(\sum_{t=\tau_{j}}^{\tau_{j+1}-1}\sum_{i=1}^{n}z_{t,i}z_{t,i}^{\top}\right)\left(\thetataulock-\tilde{\theta}_{\tau_j}\right)\\
	& \le \left(\tau_{j+1}-\tau_{j}\right)nR_{j}^{2}\epsilon^{2} \Tr\left(\Sigmainit\right)+\left(\tau_{j+1}-\tau_{j}\right)nR_{j}^{2}\left\Vert \thetataulock-\tilde{\theta}_{\tau_j}\right\Vert ^{2}.
	\end{align*}
	We arrive at our results after summing over all $j = 1,...,|\tau|$.
\end{proof}

\begin{thm}[Type II regret for BLR]
	Assume that there is an $R$ such that $\frac{1}{n(\tau_{j+1}-\tau_{j})}\sum_{t=\tau_{j}}^{\tau_{j+1}-1}\sum_{i=1}^{n}z_{t,i}z_{t,i}^{\top}\preceq R^{2}I$ for all $j\in\{1, 2, \ldots, |\boldsymbol{\tau}|\}$. It holds that
	\begin{align*}
	\Lblr - \Ldyntau & \le\frac{1}{2}\left(\thetataulock-\thetainit\right)^{\top}\Sigmainit^{-1}\left(\thetataulock-\thetainit\right) + \frac{d}{2}\log\left(\frac{d+cnTR^{2}\Tr(\Sigmainit)}{d}\right) \\
	&\quad\quad+ \frac{cnR^2}{2}\sum_{j=1}^{|\tau|}\left(\tau_{j+1}-\tau_{j}\right)\left\Vert \thetataulock-\tilde{\theta}_{\tau_j}\right\Vert^{2}.
	\end{align*}
\end{thm}

\begin{proof}
	To bound the first summand of decomposition (\ref{eq:blr_decomposition}), we use Lemmas~\ref{lemma:variational_bound_general} and \ref{lemma:KL_gauss_rand_walk} and the fact that $p_0(\taulock) = 1$ under BLR. We use Lemma \ref{lemma:blr} to bound the second summand of decomposition (\ref{eq:blr_decomposition}). Thus, we obtain
	\begin{align*}
	\Lblr - \Llocked \le \frac{1}{2}\epsilon^2 d
	&+ \frac{1}{2}\left(\tilde{\theta}_{\tau_{locked}}-\thetainit\right)^{\top}\Sigma^{-1}\left(\tilde{\theta}_{\tau_{locked}}-\thetainit\right) + \frac{d}{2} - d \log(\epsilon) + \frac{cnTR^2}{2} \epsilon^2 \Tr(\Sigmainit) \\
	&+ \frac{cnR}{2}\sum_{j=1}^{|\tau|}\left(\tau_{j+1}-\tau_{j}\right)\left\Vert \tilde{\theta}_{\tau_{locked}}-\tilde{\theta}_{\tau,j}\right\Vert ^{2}.
	\end{align*}
	Choosing $\epsilon^2 = \frac{d}{d + cnTR^2\Tr(\Sigmainit)}$ will minimize this expression, which yields the desired conclusion.
\end{proof}

\subsection{Type II $\boldsymbol{\tau}$-regret results for MarBLR}

As before, we bound the difference in the cumulative negative log-likelihood,
$\Lbf-\Ldyntau$, by breaking it into two summands
\begin{equation}
\Lbf-\Ldyntau=\left(\Lbf-L_{Q_{\boldsymbol{\tau}',\tilde{\boldsymbol{\theta}}',\boldsymbol{\epsilon}^2 \Sigmainit}}\right)+\left(L_{Q_{\boldsymbol{\tau}',\tilde{\boldsymbol{\theta}}',\boldsymbol{\epsilon}^2 \Sigmainit}}-\Ldyntau\right).
\label{eq:type2_MarBLR_decomp}
\end{equation}
Thus the proof proceeds by comparing against an intermediary distribution $Q_{\boldsymbol{\tau}', \tilde{\boldsymbol{\theta}}',\boldsymbol{\epsilon}^2 \Sigmainit}$ defined per \eqref{eq:Qtaumusigma}, where $\boldsymbol{\tau}'$ be any subsequence of $\boldsymbol{\tau}$ with $\tau'_{1}=1$, $\tilde{\boldsymbol{\theta}}' := (\tilde{\theta}_t)_{t\in\boldsymbol{\tau}'}$, and $\boldsymbol{\epsilon}^2 = \left(\epsilon_1^2, \epsilon_2^2, \ldots, \epsilon_{|\boldsymbol{\tau}'|}^2 \right)$.
This intermediary distribution is centered around a dynamic oracle that may evolve slower than than the specified update times $\boldsymbol{\tau}$.
The final Type II regret bound will depend on $\boldsymbol{\tau}'$. 
Optimizing our choice of $\boldsymbol{\tau}'$ can lead to tighter Type Ii regret bounds, particularly when $\alpha$ in the MarBLR prior is small and $|\tau|$ is large.

We use the following lemma to bound the first summand of \eqref{eq:type2_MarBLR_decomp}.

\begin{lem}\label{lemma:KL_Q_p0}
	Consider the distribution $Q_{\boldsymbol{\tau},\boldsymbol{\mu}, \boldsymbol{\epsilon}^2 \Sigmainit}$ as defined above, and the MarBLR prior $p_0$ as defined per \eqref{eq:MarBLR_prior_theta_1} and \eqref{eq:MarBLR_prior_theta_t}.
	For any $\boldsymbol{\tau}$, $\boldsymbol{\mu}$ and $\boldsymbol{\epsilon}^2$, we have that
	
	\begin{align*}
	& KL\left(Q_{\boldsymbol{\tau},\boldsymbol{\mu}, \boldsymbol{\epsilon}^2 \Sigmainit} \mid\mid p_{0}\left(\boldsymbol{\theta}\mid\tau\right)\right) \\
	& = \frac{d}{2}\epsilon_1^2 + \frac{1}{2} \left(\mu_1 - \thetainit\right)^\top \Sigmainit^{-1} \left(\mu_1 - \thetainit\right)
	-d\log \epsilon_1 - \frac{d}{2}|\boldsymbol{\tau}| + (|\boldsymbol{\tau}| - 1) d \log \delta \\
	& + \sum_{t=2}^{|\boldsymbol{\tau}|} \left[ \frac{1}{2\delta^2} \left(d\left(\epsilon_{t-1}^2 + \epsilon_{t}^2\right) + \left(\mu_t - \mu_{t-1}\right)^\top \Sigmainit^{-1} \left(\mu_t - \mu_{t-1}\right)\right) - d\log \epsilon_t\right].
	\end{align*}
\end{lem}

\begin{proof}
	We define $\Theta_{J}$ and $\psub$ as in Lemma~\ref{lemma:KL_gauss_rand_walk}.
	We define $\Qsub$ as the distribution over $\Theta_{J}$ as defined by $Q_{\boldsymbol{\tau},\boldsymbol{\mu}, \boldsymbol{\epsilon}^2 \Sigmainit}$.
	We have that
	\begin{align}
	KL\left(Q_{\tau,(\mu_{t})_{t\in\tau}}\mid\mid p_{0}\left(\boldsymbol{\theta}\mid\tau\right)\right) &= KL\left(\Qsub\mid\mid \psub\right) \nonumber \\
	& =\int\cdots\int \Qsub(\boldsymbol{\theta}) \sum_{t=1}^{|\boldsymbol{\tau}|}\log\frac{\Qsub\left(\theta_{t}\right)}{\psub\left(\theta_{t}\mid\theta_{t-1}\right)}d\theta_{|\boldsymbol{\tau}|}\cdots d\theta_{1}, \label{eq:kl_integral}
	\end{align}
	because $\theta_{t}$ in $\Qsub$
	are jointly independent and $\theta_{t}$ in $\psub$ only depend
	on $\theta_{t-1}$.
	As such,
	\begin{align}
	KL\left(\Qsub\mid\mid \psub\right) & =\int \Qsub\left(\theta_{1}\right)\log\frac{\Qsub\left(\theta_{1}\right)}{\psub\left(\theta_{1}\right)}d\theta_{1} \label{eq:kl_summand_1}\\
	&\quad +\sum_{t=2}^{|\boldsymbol{\tau}|}\int\int \Qsub\left(\theta_{t-1},\theta_{t}\right)\log\frac{\Qsub\left(\theta_{t}\right)}{\psub\left(\theta_{t}\mid\theta_{t-1}\right)}d\theta_{t}d\theta_{t-1}. \label{eq:kl_summand_2}
	\end{align}
	The first term \eqref{eq:kl_summand_1} is the KL divergence of two multivariate Normal distributions, $N(\mu_1, \epsilon_1^2 \Sigmainit)$ and $N(\thetainit, \Sigmainit)$, and can be shown to be equal to
	\begin{equation}
	\int \Qsub\left(\theta_{1}\right)\log\frac{\Qsub\left(\theta_{1}\right)}{\psub\left(\theta_{1}\right)}d\theta_{1}
	= \frac{1}{2}\epsilon_1^{2}d+\frac{1}{2}\left(\mu_1-\thetainit\right)^{\top}\Sigmainit^{-1}\left(\mu_1-\thetainit\right)-d\log\epsilon_1-\frac{d}{2}.
	\label{eq:KL_intermed_qsub_1}
	\end{equation}
	
	Next each term in the summation of \eqref{eq:kl_summand_2} is equal to
	\begin{align}
	& \int \int \Qsub({\theta}_{t-1}, {\theta}_{t}) 
	\log \frac{\Qsub({\theta}_{t})}{\psub({\theta}_{t}\mid {\theta}_{t-1})}
	\mathrm{d}{\theta}_{t} \mathrm{d}{\theta}_{t-1} \nonumber \\
	= &
	d \log\frac{\delta}{\epsilon_t}
	+\frac{1}{2\delta^{2}}
	\int \int \Qsub({\theta}_{t-1}, {\theta}_{t}) 
	\left(
	{\theta}_{t} -  {\theta}_{t-1}
	\right)^\top
	\Sigmainit^{-1}
	\left(
	{\theta}_{t} -  {\theta}_{t-1}
	\right)
	\mathrm{d}{\theta}_{t} \mathrm{d}{\theta}_{t-1}
	-\frac{d}{2}.
	\label{eq:qq_intermed_qsub}
	\end{align}
	
	We note that under $\Qsub$ it holds that
	\begin{equation*}
	(\theta_{t}-\theta_{t-1}) \sim N\left(\mu_{t}-\mu_{t-1}, (\epsilon_{t-1}^{2}+\epsilon_{t}^{2})\Sigmainit\right).
	\end{equation*}
	
	Therefore, \eqref{eq:qq_intermed_qsub} simplifies to
	\begin{align}
	& \int \int \Qsub({\theta}_{t-1}, {\theta}_{t}) 
	\log \frac{\Qsub({\theta}_{t})}{\psub({\theta}_{t}\mid {\theta}_{t-1})}
	\mathrm{d}{\theta}_{t} \mathrm{d}{\theta}_{t-1} \nonumber \\
	= & d \log\frac{\delta}{\epsilon_t}
	+\frac{1}{2\delta^{2}}
	\left(
	d\left(\epsilon_{t-1}^{2}+\epsilon_{t}^{2}\right) + (\mu_{t}-\mu_{t-1})^\top \Sigmainit^{-1} (\mu_{t}-\mu_{t-1})
	\right)
	-\frac{d}{2}.
	\label{eq:KL_intermed_qsub_2}
	\end{align}
	
	Plugging \eqref{eq:KL_intermed_qsub_1} and \eqref{eq:KL_intermed_qsub_2} into \eqref{eq:kl_summand_1} and \eqref{eq:kl_summand_2} gives us the desired result.
\end{proof}

Next we need to bound the second summand of \eqref{eq:type2_MarBLR_decomp}.

\begin{lem}\label{lemma:Lq_minus_Ldyn}
	Suppose there is a constant $c$ that bounds the second derivative as in \eqref{eq:second_derivative_assumption}.
	Assume that there is an $R$ such that $\frac{1}{n(\tau_{j+1}-\tau_{j})}\sum_{t=\tau_{j}}^{\tau_{j+1}-1}\sum_{i=1}^{n}z_{t,i}z_{t,i}^{\top}\preceq R^{2}I$ for all $j\in\{1, 2, \ldots, |\boldsymbol{\tau}|\}$.
	Then it holds that
	\begin{align*}
	L_{Q_{\boldsymbol{\tau}',\tilde{\boldsymbol{\theta}}',\boldsymbol{\epsilon}^2 \Sigmainit}}-\Ldyntau
	& \le \frac{1}{2} c nR^{2} \sum_{j=1}^{|\boldsymbol{\tau}|} (\tau_{j + 1} - \tau_{j})
	\left( \epsilon_{k(j)}^2 \Tr(\Sigmainit) + \left\| \tilde{\theta}_{\tau'_{k(j)}}-\tilde{\theta}_{\tau_{j}} \right\|^{2} \right)
	\end{align*}
	where $k(j) := \max \{k : \tau'_k \leq \tau_j\}$.
\end{lem}

\begin{proof}
	For the ease of notation denote $\tilde{Q} := Q_{\boldsymbol{\tau}',\tilde{\boldsymbol{\theta}}',\boldsymbol{\epsilon}^2 \Sigmainit}$.
	It holds that
	\begin{align*}
	&L_{\tilde{Q}}-\Ldyntau \\
	&= \sum_{j=1}^{|\boldsymbol{\tau}|}\left(E_{\tilde{Q}}\left[\sum_{t=\tau_{j}}^{\tau_{j+1}-1}\sum_{i=1}^{n}-\log p\left(y_{t,i}\mid z_{t,i};\theta_{t}\right)+\log p\left(y_{t,i}\mid z_{t,i};\tilde{\theta}_{\tau_{j}}\right)\right]\right)
	\end{align*}
	
	Recall that for any sequence $\boldsymbol{\theta}$ drawn from $\tilde{Q}$, for any $j = 1, \ldots, |\boldsymbol{\tau}|$, the parameters $\theta_{t}$ are constant over $t=\tau'_{j}, \ldots, \tau'_{j+1}-1$.
	Taking a Taylor expansion, there exists some $\theta_{mid}$ such that
	
	\begin{align}
	\begin{split}
	-\sum_{t=\tau_{j}}^{\tau_{j+1}-1}\sum_{i=1}^{n}\log p\left(y_{t,i} \middle| z_{t,i};\theta\right) 
	&=-\sum_{t=\tau_{j}}^{\tau_{j+1}-1}\sum_{i=1}^{n}\log p\left(y_{t,i} \middle| z_{t,i};\tilde{\theta}_{\tau_{j}}\right) \\
	&\quad\quad -\left.\nabla_{\theta}\sum_{t=\tau_{j}}^{\tau_{j+1}-1}\sum_{i=1}^{n}\log p\left(y_{t,i}\middle| z_{t,i};\theta\right)\right|_{\theta=\tilde{\theta}_{\tau_{j}}}^{\top}\left(\theta-\tilde{\theta}_{\tau_{j}}\right)  \\
	&\quad\quad -\frac{1}{2}\left(\theta-\tilde{\theta}_{\tau_{j}}\right)^{\top}\left.\nabla_{\theta}^{2}\sum_{t=\tau_{j}}^{\tau_{j+1}-1}\sum_{i=1}^{n}\log p\left(y_{t,i}\middle| z_{t,i};\theta\right)\right|_{\theta=\theta_{mid}}\left(\theta-\tilde{\theta}_{\tau_{j}}\right). \label{eq:MarBLR_taylor}
	\end{split}
	\end{align}
	
	Since $\boldsymbol{\tau}'$ is a subsequence of $\boldsymbol{\tau}$, for $\boldsymbol{\theta} \sim \tilde{Q}$ we have that $\theta_t = \theta_{\tau'_{k(j)}}$ for all $t=\tau_{j},....,\tau_{j+1}-1$, where $k(j) := \max \{k : \tau'_k \leq \tau_j\}$. Thus, we can use the above decomposition to evaluate \eqref{eq:MarBLR_taylor} with $\theta_t = \theta_{\tau'_{k_j}}$ in place of $\theta$.
	
	By the definition of $\tilde{\theta}_{\tau_{j}}$, the gradient in the expression above is zero, so the second term is equal to zero. Because we assumed the second derivative was bounded by $c$ as in \eqref{eq:second_derivative_assumption}, the expression simplifies to the bound
	\begin{align*}
	-\sum_{t=\tau_{j}}^{\tau_{j+1}-1}\sum_{i=1}^{n}\log p\left(y_{t,i}\mid z_{t,i};\theta_t\right) & \le-\sum_{t=\tau_{j}}^{\tau_{j+1}-1}\sum_{i=1}^{n}\log p\left(y_{t,i}\mid z_{t,i};\tilde{\theta}_{\tau_{j}}\right) \\
	&\quad + \frac{c}{2}\left(\theta_{\tau'_{k(j)}}-\tilde{\theta}_{\tau_{j}}\right)^\top \left(\sum_{t=\tau_j}^{\tau_{j+1}-1}\sum_{i=1}^{n}z_{t,i}z_{t,i}^{\top}\right)\left(\theta_{\tau'_{k(j)}}-\tilde{\theta}_{\tau_{j}}\right).
	\end{align*}
	
	Assuming there exists some $R^2$ that satisfies the lemma assumptions, it follows that
	\begin{align*}
	&\E_{\tilde{Q}}\left[\left( \theta_{\tau'_{k(j)}}-\tilde{\theta}_{\tau_{j}} \right)^{\top} \left(\sum_{t=\tau_j}^{\tau_{j + 1} - 1}\sum_{i=1}^{n} z_{t,i} z_{t,i}^{\top}\right)\left( \theta_{\tau'_{k(j)}}-\tilde{\theta}_{\tau_{j}} \right)\right] \\
	&\le (\tau_{j + 1} - \tau_{j})nR^{2}  \E_{\tilde{Q}}\left\| \theta_{\tau'_{k(j)}}-\tilde{\theta}_{\tau_{j}} \right\|^{2} \\
	&= (\tau_{j + 1} - \tau_{j})nR^{2} \left( \epsilon_{k(j)}^2 \Tr(\Sigmainit) + \left\| \tilde{\theta}_{\tau'_{k(j)}}-\tilde{\theta}_{\tau_{j}} \right\|^{2} \right).
	\end{align*}
	
	We finish the proof by summing over $j$.
\end{proof}

We combine the results to get the following bound.

\begin{thm}[Type II regret for MarBLR] \label{thm:type2_regret_MarBLR}
	Suppose there is a constant $c$ that bounds the second derivative as in \eqref{eq:second_derivative_assumption}.
	Assume that there is an $R$ such that $\frac{1}{n(\tau_{j+1}-\tau_{j})}\sum_{t=\tau_{j}}^{\tau_{j+1}-1}\sum_{i=1}^{n}z_{t,i}z_{t,i}^{\top}\preceq R^{2}I$ for all $j\in\{1, 2, \ldots, |\boldsymbol{\tau}|\}$.
	Let $\boldsymbol{\tau}'$ be any subsequence of the sequence of shift times $\boldsymbol{\tau}$.
	Then it holds that
	\begin{align*}
	&\LMarBLR-\Ldyntau \\
	&\leq \frac{1}{2} \left(\tilde{\theta}_1 - \thetainit\right)^\top \Sigmainit^{-1} \left(\tilde{\theta}_1 - \thetainit\right)
	+ \frac{d}{2} \log\left(1 + \frac{1}{\delta^2} + \frac{cnR^2 \Tr(\Sigmainit)\left( \tau'_2 - \tau'_1 \right)}{d}\right) \\
	& + \frac{1}{2} \sum_{t=2}^{|\boldsymbol{\tau}'|} \left[ \frac{1}{\delta^2} \left(\tilde{\theta}_{\tau'_t} - \tilde{\theta}_{\tau'_{t-1}}\right)^\top \Sigmainit^{-1} \left(\tilde{\theta}_{\tau'_t} - \tilde{\theta}_{\tau'_{t-1}}\right)
	+ d \log\left(\frac{2}{\delta^2} + \frac{cnR^2 \Tr(\Sigmainit) \left(\tau'_{j+1} - \tau'_j\right)}{d}\right) \right] \\
	&  - \log p_0(\boldsymbol{\tau}') + (|\boldsymbol{\tau}'| - 1) d \log \delta
	+ \frac{1}{2} c nR^{2} \sum_{j=1}^{|\boldsymbol{\tau}|} (\tau_{j + 1} - \tau_{j})
	\left\| \tilde{\theta}_{\tau'_{k(j)}}-\tilde{\theta}_{\tau_{j}} \right\|^{2}.
	\end{align*}
\end{thm}

\begin{proof}
	By combining Lemmas \ref{lemma:variational_bound_general}, \ref{lemma:KL_Q_p0} and \ref{lemma:Lq_minus_Ldyn} we obtain the following upper bound
	\begin{equation}
	\begin{aligned}
	\Lbf-\Ldyntau
	&\leq
	\frac{d}{2}\epsilon_1^2 + \frac{1}{2} \left(\tilde{\theta}_1 - \thetainit\right)^\top \Sigmainit^{-1} \left(\tilde{\theta}_1 - \thetainit\right)
	-d\log \epsilon_1 - \frac{d}{2}|\boldsymbol{\tau}'| + (|\boldsymbol{\tau}'| - 1) d \log \delta \\
	& + \sum_{t=2}^{|\boldsymbol{\tau}'|} \left[ \frac{1}{2\delta^2} \left(d\left(\epsilon_{t-1}^2 + \epsilon_{t}^2\right) + \left(\tilde{\theta}_{\tau'_t} - \tilde{\theta}_{\tau'_{t-1}}\right)^\top \Sigmainit^{-1} \left(\tilde{\theta}_{\tau'_t} - \tilde{\theta}_{\tau'_{t-1}}\right)\right) - d\log \epsilon_t\right] \\
	&  - \log p_0(\boldsymbol{\tau}')
	+ \frac{1}{2} c nR^{2} \sum_{j=1}^{|\boldsymbol{\tau}|} (\tau_{j + 1} - \tau_{j})
	\left( \epsilon_{k(j)}^2 \Tr(\Sigmainit) + \left\| \tilde{\theta}_{\tau'_{k(j)}}-\tilde{\theta}_{\tau_{j}} \right\|^{2} \right).
	\end{aligned}
	\label{eq:upper_bound_MarBLR_intermed}
	\end{equation}
	
	We minimize the upper bound with respect to $\left(\epsilon_j\right)_{j=1,2,\ldots,|\boldsymbol{\tau}'|}$.
	
	For $j=1$, $\epsilon_{1}$ only contributes to the above bound through the terms
	\begin{equation}
	\frac{1}{2}\left( d + \frac{d}{\delta^2} + cnR^2 \Tr(\Sigmainit)\left( \tau'_2 - \tau'_1 \right) \right) \epsilon_1^2 - d \log \epsilon_1.
	\label{eq:eps_1_term}
	\end{equation}
	
	For $j=2,...,|\tau'|-1$, $\epsilon_{j}$ only contributes to the bound through the terms
	\begin{equation}
	\frac{1}{2} \left[ \frac{2d}{\delta^2} + cnR^2 \Tr(\Sigmainit) \left(\tau'_{j+1} - \tau'_j\right) \right] \epsilon_j^2 - d \log \epsilon_j.
	\label{eq:eps_j_term}
	\end{equation}
	
	For $j=|\boldsymbol{\tau}'|$, $\epsilon_{|\boldsymbol{\tau}'|}$ only contributes to the bound through the terms
	\begin{equation}
	\frac{1}{2} \left[ \frac{d}{\delta^2} + cnR^2 \Tr(\Sigmainit) \left(\tau'_{|\boldsymbol{\tau}'|+1} - \tau'_{|\boldsymbol{\tau}'|}\right) \right] \epsilon_{|\boldsymbol{\tau}'|}^2 - d \log \epsilon_{|\boldsymbol{\tau}'|}.
	\label{eq:eps_tau_term}
	\end{equation}
	
	It follows that the upper bound is minimized for
	\begin{align}
	\epsilon_1^2 &= \frac{d}{d + \frac{d}{\delta^2} + cnR^2 \Tr(\Sigmainit)\left( \tau'_2 - \tau'_1 \right)}, \label{eq:min_eps_1} \\
	\epsilon_j^2 &= \frac{d}{\frac{2d}{\delta^2} + cnR^2 \Tr(\Sigmainit) \left(\tau'_{j+1} - \tau'_j\right) }, \quad\forall j \in \left\{ 2, 3, \ldots, |\boldsymbol{\tau}'|-1 \right\}, \label{eq:min_eps_j} \\
	\epsilon_{|\boldsymbol{\tau}'|} &= \frac{d}{\frac{d}{\delta^2} + cnR^2 \Tr(\Sigmainit) \left(\tau'_{|\boldsymbol{\tau}'|+1} - \tau'_{|\boldsymbol{\tau}'|}\right) }. \label{eq:min_eps_tau}
	\end{align}
	
	Note that the upper bound \eqref{eq:upper_bound_MarBLR_intermed} is a sum of the terms \eqref{eq:eps_1_term}, \eqref{eq:eps_j_term} repeated once for each $j \in \left\{ 2, 3, \ldots, |\boldsymbol{\tau}'|-1 \right\}$, \eqref{eq:eps_tau_term}, and the following remaining terms
	\begin{align*}
	&\frac{1}{2} \left(\tilde{\theta}_1 - \thetainit\right)^\top \Sigmainit^{-1} \left(\tilde{\theta}_1 - \thetainit\right)
	- \frac{d}{2}|\boldsymbol{\tau}'| + (|\boldsymbol{\tau}'| - 1) d \log \delta  - \log p_0(\boldsymbol{\tau}') \\
	& + \sum_{t=2}^{|\boldsymbol{\tau}'|} \frac{1}{2\delta^2} \left(\tilde{\theta}_{\tau'_t} - \tilde{\theta}_{\tau'_{t-1}}\right)^\top \Sigmainit^{-1} \left(\tilde{\theta}_{\tau'_t} - \tilde{\theta}_{\tau'_{t-1}}\right)
	+ \frac{1}{2} c nR^{2} \sum_{j=1}^{|\boldsymbol{\tau}|} (\tau_{j + 1} - \tau_{j})
	\left\| \tilde{\theta}_{\tau'_{k(j)}}-\tilde{\theta}_{\tau_{j}} \right\|^{2}.
	\end{align*}
	Plugging in \eqref{eq:min_eps_1}, \eqref{eq:min_eps_j} and \eqref{eq:min_eps_tau}, we get the desired bound.
\end{proof}

\begin{figure}
	\centering
	\begin{subfigure}{0.49\linewidth}
		\begin{center}
		\includegraphics[width=0.6\linewidth]{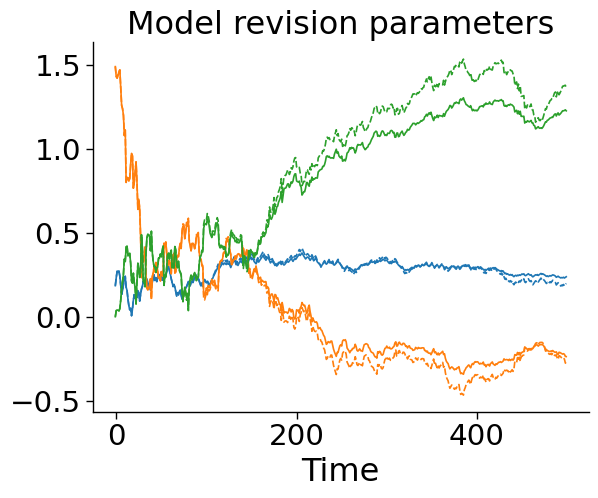}
		\end{center}
		\caption{\texttt{Initial Shift}, \texttt{All-Refit}}
		\label{fig:iid_good}
	\end{subfigure}
	\begin{subfigure}{0.49\linewidth}
		\includegraphics[width=0.6\linewidth]{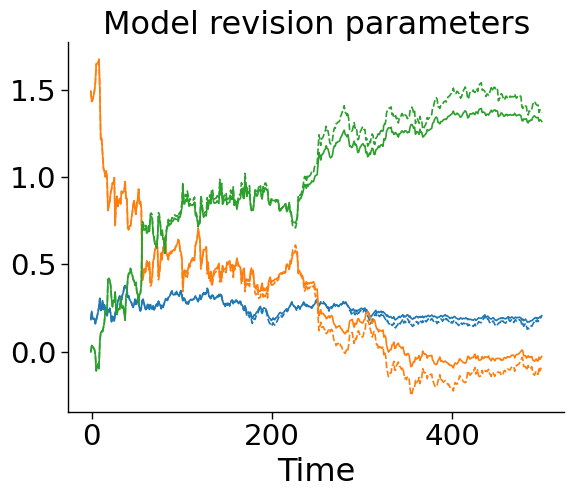}
		\caption{\texttt{Decay}, \texttt{All-Refit}}
		\label{fig:shift_good}
	\end{subfigure}
	
	\begin{subfigure}{0.49\linewidth}
	\begin{center}
		\includegraphics[width=0.6\linewidth]{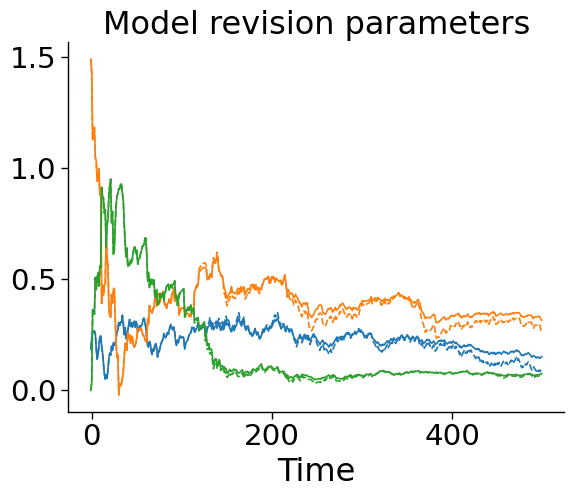}
	\end{center}
		\caption{\texttt{Initial Shift}, \texttt{Subset-Refit}}
		\label{fig:iid_bad}
	\end{subfigure}
	\begin{subfigure}{0.49\linewidth}
		\includegraphics[width=0.6\linewidth]{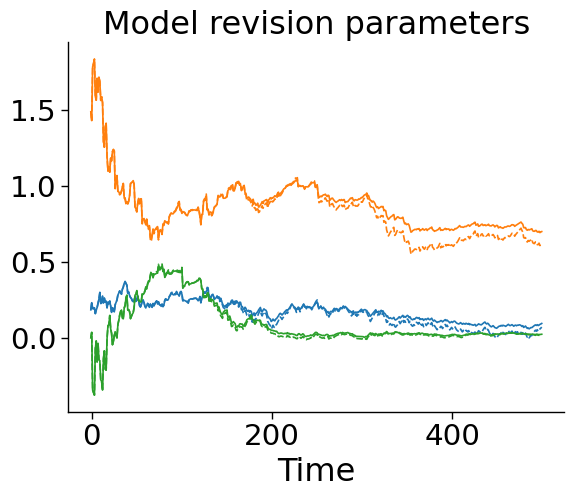}
		\includegraphics[width=0.35\linewidth]{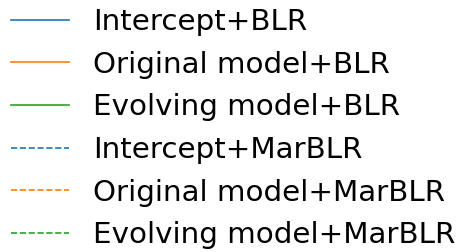}
		\caption{\texttt{Decay}, \texttt{Subset-Refit}}
		\label{fig:shift_bad}
	\end{subfigure}
	\caption{Evolution of the estimated intercepts and coefficients by BLR and MarBLR when combining the original model with an evolving prediction model (Scenario 3).
		Data is simulated to be stationary over time after an initial shift (\texttt{Initial Shift}) and nonstationary such that the original model decays in performance over time (\texttt{Decay}). 
		Underlying prediction model is updated by continually refitting on all previous data (\texttt{All-Refit}) or refit on the most recent subset of data (\texttt{Subset-Refit}).
	}
	\label{fig:refitting_thetas}
\end{figure}

\begin{figure}
	\centering
	\begin{subfigure}{\linewidth}
		\centering
		\includegraphics[width=0.5\linewidth]{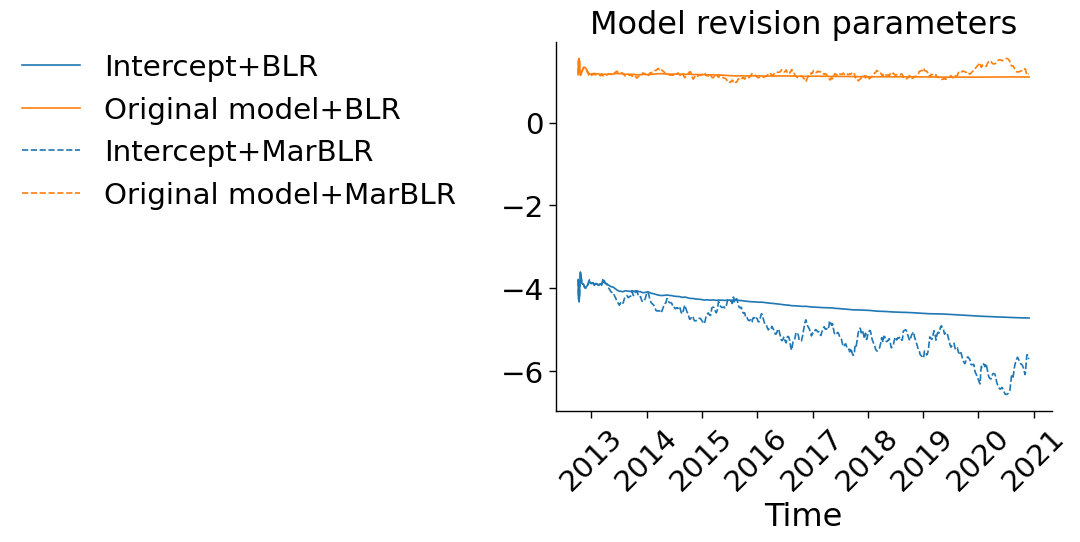}
		\caption{Online recalibration of a fixed prediction model}
		\label{fig:online_recalib_copd}
	\end{subfigure}	
	\begin{subfigure}{\linewidth}
		\centering
		\includegraphics[width=0.5\linewidth]{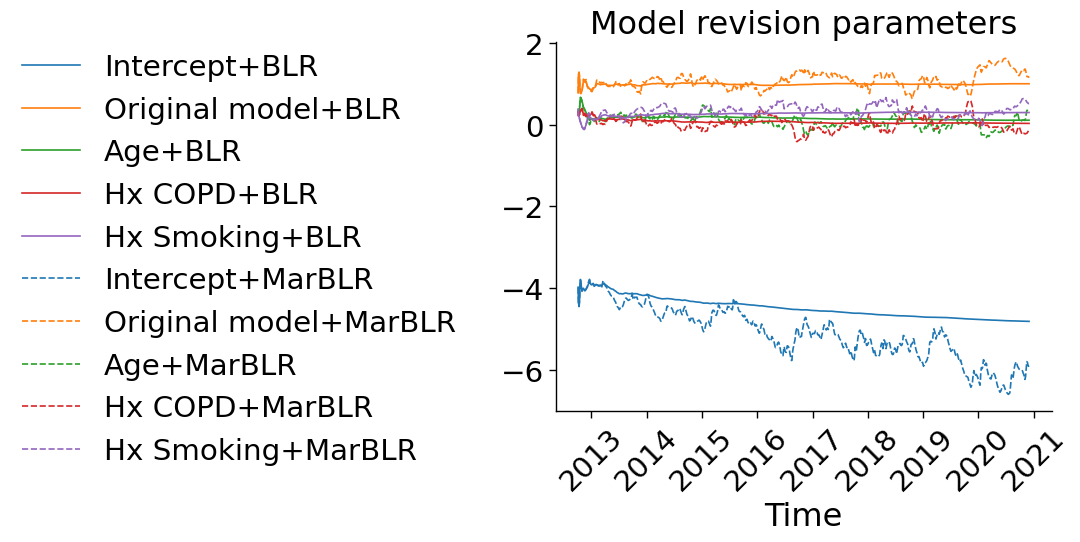}
		\caption{Online logistic revision with respect to a fixed prediction model and patient variables}
		\label{fig:online_linear_copd}
	\end{subfigure}	
	\begin{subfigure}{\linewidth}
		\centering
		\includegraphics[width=0.5\linewidth]{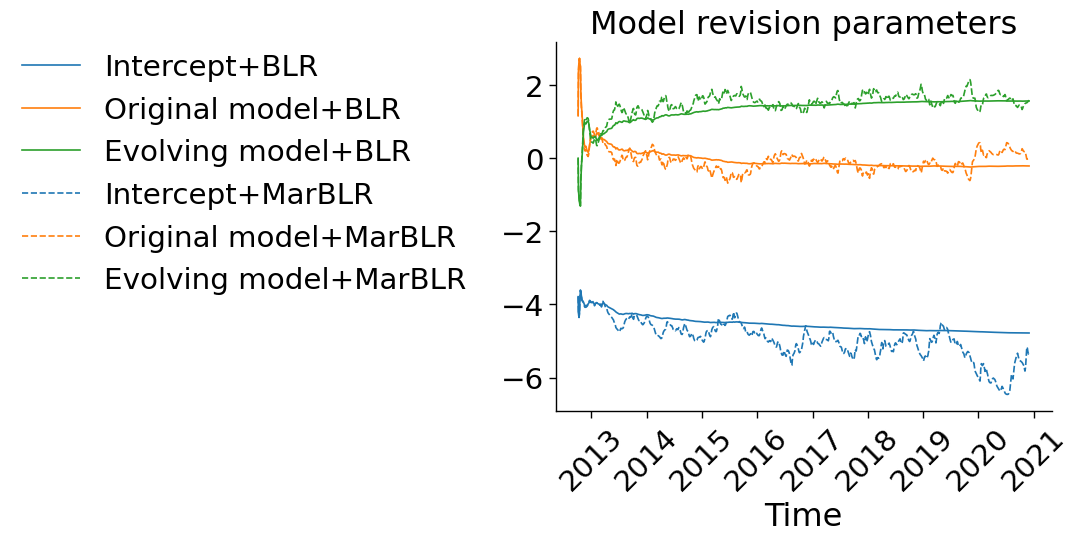}
		\caption{Online ensembling of the original and continually-refitted prediction models}
		\label{fig:online_refit_copd}
	\end{subfigure}
	\caption{Evolution of the estimated intercepts and coefficients for online recalibration and revision of a fixed COPD risk prediction model (a and b, respectively) and online reweighting for fixed and continually-refitted (evolving) COPD risk prediction models using BLR and MarBLR.
	}
	\label{fig:copd_thetas}
\end{figure}

\pagebreak

\bibliographystyle{plainnat}
\bibliography{main}

\end{document}